\documentclass[accepted]{uai2021} %
\usepackage[table]{xcolor}

\usepackage[american]{babel}

\usepackage{natbib} %
\usepackage{mathtools} %
\usepackage{booktabs,paralist,hhline,colortbl} %

\definecolor{Maroon}{RGB}{204, 102, 0}

\definecolor{rulecolor}{rgb}{0.0, 0.06, 0.54}
\definecolor{tableheadcolor}{rgb}{0.74, 0.83, 0.9}
\definecolor{bluecolor}{rgb}{0.74, 0.83, 0.9}

\usepackage{tikz} %

\usepackage{url}
\usepackage{hyperref}
\usepackage{graphicx}
\usepackage{dsfont}
\usepackage{amsmath}
\usepackage{amssymb}

\usepackage{empheq}

\usepackage[linesnumbered,ruled,vlined]{algorithm2e}
\usepackage{graphicx}
\usepackage{subcaption}
\usepackage{listings}

\usepackage{bbm}
\usepackage{multirow}
\usepackage{rotating}
\usepackage{multicol}
\usepackage{amsthm,comment}
\newtheorem{theorem}{Theorem}
\newtheorem{lemma}[theorem]{Lemma}
\theoremstyle{remark}
\newtheorem*{remark}{Remark}
\newcommand{\1}[2][]{\ensuremath{\mathds{1}_{#1}\left\{#2\right\}}}
\newcommand{\obs}[1]{\ensuremath{{#1}^\text{obs}}}
\newcommand{\z}{\ensuremath{\boldsymbol{z}}}
\newcommand{\mz}{\ensuremath{\mathbf{z}}}

\renewcommand{\cite}[1]{\citep{#1}}

\usepackage{tcolorbox}

\newtcolorbox{mybox}[3][]
{
  colframe = #2!15,
  colback  = #2!10,
  coltitle = #2!10!black,  
  title    = {#3},
  boxsep   = 0.25pt,
  left     = 0.5pt,
  right    = 0.5pt,
  top      = 0pt,
  bottom   = 0pt,
  width=\linewidth,
  #1,
}

\usepackage{hyperref}

\title{Mixed Effects Neural ODE:\\A Variational Approximation for Analyzing the Dynamics of Panel Data}

\author[1,2]{Jurijs Nazarovs}
\author[3]{Rudrasis Chakraborty}
\author[2]{Songwong Tasneeyapant}
\author[4]{Sathya N. Ravi$^\ast$}
\author[2]{Vikas Singh\thanks{Corresponding authors: Ravi and Singh.}}

\affil[1]{%
    Department of Statistics, 
    University of Wisconsin Madison
}
\affil[2]{%
    Department of Biostatistics \& Med. Info., 
    University of Wisconsin Madison
}
\affil[3]{
    Amazon Lab 126
}
\affil[4]{
    Department of Computer Science, 
    University of Illinois at Chicago
}
\begin{document}
\maketitle

\begin{abstract}
Panel data involving longitudinal measurements of the same set of participants taken over multiple time points is common in studies to understand childhood development and disease modeling. 
Deep hybrid models that marry the predictive power of neural networks with physical 
simulators such as differential equations, 
are starting to drive 
advances in such applications. The task of modeling 
not just the observations but the hidden dynamics that are captured by the measurements poses interesting statistical/computational questions.  
We propose a probabilistic model called ME-NODE to incorporate (fixed + random) mixed effects for analyzing such panel data.  
We show that our model can be derived using smooth approximations of SDEs provided by the Wong-Zakai theorem. We then derive Evidence Based Lower Bounds for ME-NODE, and develop (efficient) training algorithms using MC based sampling methods and numerical ODE solvers. 
We demonstrate ME-NODE's utility on tasks
spanning the spectrum from simulations and toy data to real longitudinal 3D imaging data from an Alzheimer's disease (AD) study, 
and study its performance in terms of accuracy of reconstruction 
for interpolation,  
uncertainty estimates and personalized prediction.

\end{abstract}

\section{Introduction}

Observational studies in the social and health sciences 
often involve acquiring 
repeated measurements \textbf{over time} for participants/subjects. 
If most participants stay enrolled, 
we can consider each row in the {\em panel} to correspond to 
longitudinal measurements or records of an individual, at regularly or irregularly sampled time points. Modeling 
development or growth trends while accounting for \textit{variability within and across individuals} %
leads to the need for statistical models for analysis of such ``panel data'' \citep{kreindler2006effects, katsev2003hurst}.

The modeling of temporal processes can be set up as a regression task, where a  function (with unknown parameters) that is plausible for the domain  is estimated using the observed longitudinal data samples. 
Apart from splines and tools from functional data analysis, a common alternative is to use differential equations \citep{chen2008efficient, liang2013parameter, fang2011two}, which provides expressive power and many computational tools developed over decades. However, differential equations do not directly account for variability within and across subjects -- hallmark features of panel data. To capture some of these characteristics, the widely used Auto-Regressive Models (ARMA) literature incorporates white noise type functions in differential equation models, leading to various forms of stochastic differential equation (SDE) \citep{brockwell2001continuous, macurdy1982use, hedeker2006longitudinal},  
\begin{equation}
    z_{t}=f_{\mu}(z, t)dt+L_{\Sigma}(z, t) \circ d \beta(t)
    \label{eq:strat}
\end{equation}
where $z_t \in \mathbf{R}$, $f$, $L$ denote the drift and noise sensitivity functions with unknown parameters $\mu$ and $\Sigma$ respectively. 

Suppose we are given a set of (partial) measurements 
$\{z_{t_k}\}$  at certain $K$ time points $\{t_k\}_{k=1}^K$, and an efficient numerical scheme to simulate the SDE in \eqref{eq:strat}. Then, the unknown parameters $\mu$ and $\Sigma$ can be found by simply maximizing the likelihood function $p\left(z_t| \mu;\Sigma\right)$. 
In principle, it is straightforward to extend the model in \eqref{eq:strat}  to high dimensional $z_t$. However, there are two main technical challenges in this setting with stand-alone likelihood based methods: \begin{inparaenum}
[\bfseries (i)]
\item such an approach requires a large number of longitudinal measurements which is often infeasible, especially in the applications that motivate our work \citep{marinescu2018tadpole}; 
\item numerical schemes to simulate nonlinear SDEs in high dimensions are quite involved.
\end{inparaenum}
{\color{black}Indeed, these issues become pronounced when we assume that the observed data is not $z$ but $x$ which are actually measurements governed by a process that reflects the dynamics $z$.
}

The literature provides a principled way, 
called {\bf Bayesian filtering}, 
to tackle the problem described above, see \citep{sarkka2014lecture}, Ch 7. 
Let us consider that the object or measurement $x$ is evolving as non-linear function $D$ from a {\em latent} measure of dynamics/progression $z_t$ \citep{pierson2019inferring, hyun2016stgp, whitaker2017bayesian}:
    $x_t = D(z_t) + \varepsilon_t.$
So, it is natural to think of the observable $x$ of an unknown dynamics $z$ -- which we 
can call the ``latent'' representation.
The goal in Bayesian filtering, which aligns 
nicely with our task, 
is to 
compute,  $p(z|x_{t_1},...,x_{t_K})$.
Interestingly, under some assumptions,
closed form solutions for the posterior distributions $p(z|x_{t_1},...,x_{t_K})$ are available, see Chapter 10 in \citep{sarkka2019applied}.
However, these assumptions are hard to verify in general, and direct utility of these approaches for modern applications is not obvious.  

{\bf Main ideas/contributions.} 
The most important takeaway from the description above 
is not the mechanics of {\em how} Bayesian filtering is carried 
out in practice, rather, {\em what} it seeks to 
estimate. 
If we focus on the key object of interest -- the conditional distribution -- we realize that 
recent works in machine learning do provide a recipe that exploits the universal approximation properties of neural networks to represent fairly 
complex conditional distributions. 
In this case, the parameters are simply trained using off-the-shelf procedures, and DE numerical solvers are required for \eqref{eq:strat}.
While both ODE and SDE solvers are available, SDE is  typically less efficient \citep{li2020scalable, liu2019neural}. Notice that when $L_{\Sigma}\equiv 0$, that is, the observable $x$ follows an ODE and corresponding solvers can be applied, the approach that would instantiate this idea has already been successfully tried in \citep{yildiz2019ode2vae}. However, given the Bayesian filtering motivation above, is there a way to utilize ODE solvers, while preserving stochastic nature of  \eqref{eq:strat}?

Our development begins by rewriting the noise term in \eqref{eq:strat} with a series of basis functions with standard normal coefficients. We show that this modification enables incorporating random effects in our predictions -- which appropriately models the variability of the data, the requirement for successful analysis of panel data. More importantly, using our approach, we show that for a special class of latent SDE based models, which we will refer as Mixed Effects Neural ODE (ME-NODE), the parameters of the underlying neural network can be trained efficiently {\em without}  backpropagating through any SDE solvers.  %
To achieve this, we derive the Evidence Lower Bound loss -- where widely available libraries for numerical ODE solvers -- are sufficient and directly applicable. We show applications to brain imaging 
where our formulation can provide personalized prediction 
together with uncertainty, a feature of Bayesian methods.

\section{Background/notation}\label{sec:background}
\begin{figure*}[t]
\centering
\includegraphics[width=\textwidth]{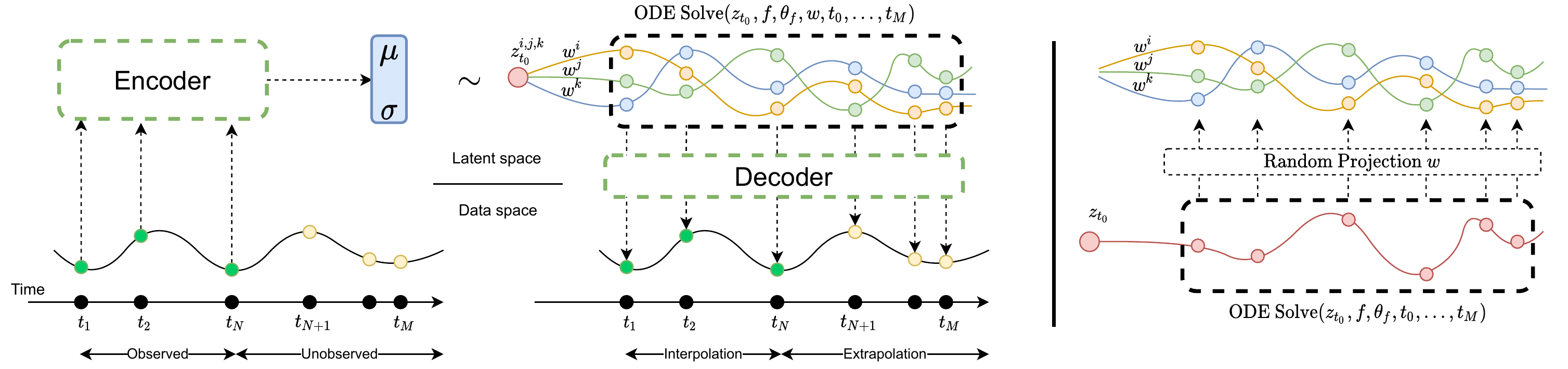} 
\caption{\footnotesize Structure of the model. First, encoder is applied to temporal data of an subject to generate initial point of the trajectory in latent space; Second, ME ODE solver is used to generate trajectory from the specified initial point; Last, decoder is used to map latent space ODE stages into observed values. On the right side we show how ME ODE can be viewed as random projection of trajectory of standard ODE, where trajectory is defined by random effect $\mathbf{w}$. 
\label{fig:model}}
\end{figure*}

In this section we present the notations and some basic concepts we use in the paper.
\paragraph{Notation.} 
For a time-varying vector    $\z=(z_0,\ldots, z_n)$, (with $n$ time points)  we denote a vector without the $j$-th component as a
$\z_{-j}=(z_0, \ldots,z_{j-1}, \quad z_{j+1}, \ldots z_n)$.
Often each time point $z_t$ is vector valued, and in the rest of the paper we denote it by $\mathbf{z}_t \in \mathbf{R}^p$, where $p$ is the number of variables. Thus, to denote a time-varying vector where each time point represents a $p$ dimensional vector we use $\z=(\mathbf{z}_0,\ldots, \mathbf{z}_n)$, where each $\left\{\mathbf{z}_t\right\}_{t=1}^n \subset  \mathbf{R}^p$. 
We denote the indicator function as \1[\mathbf{x}]{\mathbf{y}}:  it is $1$ if $\mathbf{y}=\mathbf{x}$, and $0$ otherwise.
We refer to a general form of ODE as $\dot{\mz}_t=h(\mz_t)$, where $h(\mz_t)$ defines a  ``trajectory'' and depends on the current value of the process at time $t$.
Without loss of generality, in this section we assume that the DE is defined on $\mathbf{R}$.

\paragraph{Smooth Approximations of SDE.} Consider the standard form of the Stratonovich SDE given by \eqref{eq:strat}. By the Wong and Zakai theorem \citep{hairer2015wong}, the solution to \eqref{eq:strat} can be approximated asymptotically ($N \rightarrow \infty$) by the solution of the following equation:
\begin{equation}
    \dot{z}_{t}=f(z, t)+L(z, t) \sum_{n=1}^{N} b_{n} \varphi_{n}(t),
\end{equation}
where $b_n \stackrel{i.i.d.}{\sim} \mathcal{N}(0, 1)$ and $\{\varphi_n\}$ are a suitable set of basis functions.
Based on mild simplifying assumptions, %
we get
\begin{equation}
    \dot{z}_{t}=f(z, t)+g(z, t) b
    \label{eq:strat_approx}
\end{equation} 
is an approximation of Stratonovich's SDE in Equation \eqref{eq:strat}, where $b\sim\mathcal{N}(0, 1)$.

\begin{proof}
Assume that $\exists g(z, t)<\infty$ and $\forall n$, $\exists$  $\sigma_n(z, t)<\infty$ such that
\begin{equation}
    \varphi_{n}(t) L(z, t)=\frac{\sigma_n(z, t)}{\sqrt{\sum_{n=1}^{N} \sigma_{n}^{2}(z, t)}} g(z, t).\quad \mbox{Then,}
\end{equation}
$$\sum_{n=1}^N \varphi_{n}(t) L(x, t)=g(z, t) \frac{1}{\sqrt{\sum_{n}^{N} \sigma_n^{2}}} \sum_{n=1}^{N} \xi_{n}(z, t),$$ where $\xi_n(z, t) = b_n\sigma_n(z, t) \sim \mathcal{N}(0,\sigma_n(z, t))$.
If for some $\delta>0$, 
\begin{equation}
    \frac{1}{\left(\sqrt{\sum_{n=1}^N \sigma_{n}^{2}}\right)^{2+\delta}} \sum_{n=1}^{N} \mathbb{E}\left(\left|\xi_{n}-\mu_{n}\right|^{2+\delta}\right)  \stackrel{N \rightarrow \infty}{\rightarrow} 0,
\end{equation}
then according to Lyapunov's Central Limit Theorem 
$\frac{1}{\sqrt{\sum_{n}^{N} \sigma_n^{2}}} \sum_{n=1}^{N} \xi_{n}(z, t) \rightarrow \mathcal{N}(0, 1)$. 
Then,
$\dot{z}_{t}=f(z, t)+g(z, t) b$ is an approximation of Stratonovich's SDE \eqref{eq:strat},
where $b\sim\mathcal{N}(0, 1)$.
\end{proof}

This means that under standard second moment conditions on the stochastic part, the  solution of \eqref{eq:strat_approx} can be seen as an approximation of the solution to Stratonovich's SDE in Equation \eqref{eq:strat}.
The benefits of this simplified form are two-fold: \begin{inparaenum}[\bfseries (a)] 
\item since $z$ is a random variable, we can incorporate uncertainty within the ODE similar to a SDE,  
\item for a given $z$,  the trajectory in  \eqref{eq:strat_approx} can be modeled using an ODE and the associated computational benefits become 
available. 
\end{inparaenum}
Concurrently, \citep{hodgkinson2020stochastic} showed that it is possible to generalize this result to a more general class of SDEs using tools from the theory of rough paths.
We must note (discussed briefly later), that for specific choice of $f$ and $g$, the RHS of Equation \eqref{eq:strat_approx} is well known in statistics and machine learning as a \textbf{Mixed Effects model} \citep{hyun2016stgp}. Therefore, 
we will refer to our model as a mixed effects model and 
use the relevant terminology from this literature whenever possible. We hope that this will make the presentation more accessible and clarify that our scheme is {\em not a  general purpose replacement to a}\ deep neural network based SDE solver.

\paragraph{Mixed effects model.} 
Assuming individuals/groups are denoted by $i$, a nonlinear mixed effects (ME) model \citep{demidenko2013mixed} can be written as:
\begin{equation}
\begin{array}{l}
\boldsymbol{\phi}^{i}=\nu\left(X^{i} \boldsymbol{\beta}+U^{i} \mathbf{b}^{i}\right)+\boldsymbol{\epsilon}^{i},
\end{array}
\label{eq:nonlinear_me}
\end{equation}
where $X^i \in \mathbf{R}^{n\times m}$ is a matrix of covariates where $n$ and $m$ are the number of observations and variables respectively. Here, $\nu$ is a non-linear (vector-valued) function,
$\boldsymbol{\beta}\in \mathbf{R}^m$ is a vector of {\em fixed} effects,
$\mathbf{b}^{i} \sim \mathcal{N}\left(\mathbf{0}, \Sigma_{b}\right)$ is a vector of {\em random} effects, 
$U^{i}$ is a design matrix (modeling choice) for random effects, $\boldsymbol{\phi}^{i} \in \mathbf{R}^n$ is the  response variable and 
$\boldsymbol{\epsilon}^{i} \sim \mathcal{N}\left(\mathbf{0}, \Sigma_{\epsilon^{i}}\right)$ represents a noise term.

{\em Task.} Our goal is to learn a latent representation of a time-varying physical process/dynamics $z$, and distribution $p(z|x_1, \ldots, x_{t_k})$. Based on the Bayesian filtering discussion above, we  focus on variational approximation techniques. 

\paragraph{Learning  latent representations with a VAE.}
Variational auto-encoders (VAE) \citep{kingma2013auto} enable learning a probability distribution on a latent space. Then, we can draw samples in the latent space -- and the decoder can generate samples in the space of observations. In practice, the parameters of the latent distribution are learned by maximizing the \textit{evidence lower bound} (ELBO) of the intractable likelihood:
\begin{equation}\label{eq:elbo}
    \log p\left(\mathbf{x}\right) \geq 
    -KL\left(q(\mathbf{z}) \| p(\mathbf{z})\right)+\mathbb{E}_{q(\mathbf{z})}\left[\log p\left(\mathbf{x}\mid \mathbf{z}\right)\right]
\end{equation}
where $\mathbf{z}$ is a sample in the latent space from the approximate posterior distribution $q(\mathbf{z})$, with a prior $p\left(\mathbf{z}\right)$, and $\mathbf{x}$ is a reconstruction of a sample (e.g., an image) with the likelihood $p\left(\mathbf{x} \mid \mathbf{z}\right)$. A common choice for $q$ is $\mathcal{N}(\boldsymbol{\mu}, \Sigma)$, where $\boldsymbol{\mu}$ and $\Sigma$ are trainable parameters \citep{kingma2013auto}.
\section{Mixed Effects Neural ODE}

Given a latent representation $\z$ of a  time-varying process with $\z=(\mathbf{z}_0,\ldots, \mathbf{z}_n)$ and $\mathbf{z}_t \in \mathbf{R}^p$, we now model the latent representation $\z$ as a mixed effects neural ODE. We will assume each of the $p$ variables to be independent, so we will seek to learn $p$ mixed effects models. Without any loss of generality, below we assume $z\in \mathbf{R}$ to denote the latent representation of the time-varying process at a time point $t$.

\paragraph{Modeling random effect $\mathbf{b}^i$ in a network $\Gamma$.}
Mixed effects in the context of ODE is a well studied topic in longitudinal data analysis literature \citep{wang2014estimating, liang2013parameter}. 
Formally, for subject $i$, given $z^i$, $\mathbf{b}^i$, and (population level) fixed effects $\boldsymbol{\beta}$, we assume that there exist a smooth function $h$  such that, 
\begin{equation}
\dot{z}^{i}=h(z^{i}, \boldsymbol{\beta}, \mathbf{b}^{i}).
\label{eq:z_h}
\end{equation}
Due to the universal approximation properties of neural networks  \citep{zhou2020universality}, 
such models are a sensible choice to express the nonlinear function $h$. 
Recall the non-linear mixed effects model from \eqref{eq:nonlinear_me}, and let us model $$h(z^{i}, \boldsymbol{\beta}, \mathbf{b}^{i})=\nu\left(\eta(z^{i}) \boldsymbol{\beta}+U^{i} \mathbf{b}^{i}\right).$$ Here, $\eta: \mathbf{R}^n \rightarrow \mathbf{R}^{n\times m}$ is a non-linear function, with $n$ and $m$ being the number of observations and variables respectively and $\boldsymbol{\beta}\in \mathbf{R}^m$, $U^i\in \mathbf{R}^{n\times m}$. With a choice of $U^i = \eta(z^i)$, %
we can model $h(z^{i}, \boldsymbol{\beta}, \mathbf{b}^{i})=\Gamma\left(z^{i}\right) \left(\boldsymbol{\beta} +  \mathbf{b}^{i}\right)$, where $\Gamma$ is a neural network with $\Gamma\left(z^{i}\right)\in \mathbf{R}^{n\times m}$. %
Now, we can derive the expressions for representing mixed effects in ODE, parameterized by a neural network as 
\begin{equation}
h(z_t^{i}, \boldsymbol{\beta}, \mathbf{b}^{i})= \Gamma\left(z_t^{i}\right)\mathbf{w}^i, 
\label{eq:lat_me}
\end{equation}
where $\mathbf{w}^i\sim \mathcal{N}(\boldsymbol{\beta}, \Sigma_{b})$ is a mixed effect for subject $i$. This can be thought of as a projection from $\mathbf{R}^m$ to $\mathbf{R}$ along the direction given by $\mathbf{w}^i\in \mathbf{R}^m$.

\begin{remark}
Observe the difference between standard SDE in \eqref{eq:strat} -- where the noise is added at each step $t$ -- and our formulation,  where $\mathbf{w}^{i}$ (or $\mathbf{b}^i$) is sampled once for subject $i$ and completely defines the trajectory through $h(z^{i}, \boldsymbol{\beta}, \mathbf{b}^{i})$ for all steps of time $t$. This is {\em crucial} from the computational perspective: with this strategy, we can simply apply existing ODE solvers whereas backpropagating through a blackbox SDE requires specialized solutions, which are typically slower \citep{li2020scalable, liu2019neural}.
\end{remark}

\paragraph{Initializing ODE $h(z^{i}, \boldsymbol{\beta}, \mathbf{b}^{i})$ with an encoder $E$.}
Often in real-world analysis tasks involving panel data, the initial point of the process $z_0$ is not observed. While it can be learned as a parameter \citep{huang2008modeling}, it is desirable to also provide uncertainty pertaining to the learned $z_0$. For this reason, we learn the distribution $q(z_0)=\mathcal{N}(\mu, \sigma)$, by training an encoder $E$ to map observed data $\boldsymbol{x}=(\mathbf{x}_1,\ldots,\mathbf{x}_n)$ (at all $n$ time points)  to parameters of $q(z_0)$, $\mu$ and $\sigma$:
\begin{equation}
    (\mu, \sigma) = E(\boldsymbol{x}) %
    \label{eq:enc}
\end{equation}
We use $q(z_0)$ to sample initial points of ODE $z_0$. 

\paragraph{Mapping $z$ to $x$ via decoder $D$.}
Given the latent representation $z$ and the non-linear function $D$, which can recover the output $\mathbf{x}^i_t$, for a subject $i$ at time point $t$, we can model the output of the dynamic process (e.g., in our application, a brain image) $\mathbf{x}^i_t$ as a non-linear transformation of the latent measure of progression $z^i_t$:
\begin{equation}
\mathbf{x}_t^{i} = D(z_t^{i}) + \boldsymbol{\epsilon}_t,
\label{eq:dec}
\end{equation}
where $\boldsymbol{\epsilon}_t$ is measurement error at each time point.
This idea has been variously used in the literature \citep{pierson2019inferring, hyun2016stgp, whitaker2017bayesian}.

\paragraph{The final model.}
Combining Equations  \eqref{eq:enc}, and \eqref{eq:dec} we obtain our Mixed Effects Neural ODE model in  \eqref{eq:me_ode} and illustrated in Figure~\ref{fig:model}. 

\begin{equation}
\left[
\begin{array}{l}
z_{0}^{i} \sim \mathcal{N}(\mu, \sigma)\text{, where } \mu, \sigma = E(\boldsymbol{x}^i)\\
\mathbf{w}^{i} = \boldsymbol{\beta} + b^{i} \sim \mathcal{N}(\boldsymbol{\beta}, \Sigma_{b})\\
\dot{z}^i_t=\Gamma\left(z^{i}_t\right)\mathbf{w}^{i}\\
\mathbf{x}^i_t = D(z^i_t) + \boldsymbol{\epsilon}_t
\end{array}
\right.
\label{eq:me_ode}
\end{equation}

{\em Synopsis.} Here, for subject $i$, we use the encoder $E$ to map the observed data $\boldsymbol{x^i}$ to parameters of the distribution of ODE initialization $z_0^i$. Then, we parameterize the derivative of ODE $\dot{z}_t$ with a neural network $\Gamma(z^i_t)$ and mixed effects $\mathbf{w}^i$, and use $D$ to map solution of ODE to the original space.

{\bf Structure of the latent space.} When the latent space $z$ can be embedded in a low dimensional space, it is natural to ask whether simulating the ODE $\dot{z}$ in \eqref{eq:me_ode} can be accomplished efficiently.
The following result shows a link between our model in \eqref{eq:me_ode} and random projection ideas \citep{vempala2005random}. 
Moreover, in contrast to Neural ODE \citep{chen2018neural}, random projections allow using high dimensional representations,  $\Gamma(z^i_t)$, and mapping it back to $\mathbf{R}$ using random projections.  
This provides expressive power but also approximately preserves the distance (using JL lemma). 

\begin{lemma}[Random projection]
With a certain choice of $h(z_t)$ in the ODE formulation and given a mixed effect $\mathbf{w}$ with a choice of approximate posterior $q(\mathbf{w})$ as a Normal distribution, the solution to ME Neural ODE (Equation \eqref{eq:me_ode}) is a random projection of a solution to Neural ODE  \citep{chen2018neural}, Figure~\ref{fig:model} (right part). %
\label{lemma:randproj}
\end{lemma}
\begin{proof}
Let us use the following notations,
\begin{compactenum} [\bfseries (a)]
\item $d z_{t} =f\left(z_{t}\right) d t$ defines trajectory of Neural ODE setup;
\item $d \widetilde{z_{t}}=\Gamma\left(\widetilde{z_{t}}\right) \cdot w d t$ defines trajectory of our ME setup;
\item $\widetilde{z_{t}}= z_t \cdot w$: random projection of  $z_t$.
Then we have,
\end{compactenum}
\begin{align*}
d \widetilde{z}_{t} =w \cdot dz_t
= w \cdot f\left(z_{t}\right) d t
&=f\left(\frac{\widetilde{z_{t}}}{w}\right) \cdot w d t\\* 
&=  \Gamma(\widetilde{z_t}) \cdot w db
\end{align*}
If $\Gamma\left(\widetilde{z}_{t}\right)=f\left(\frac{\widetilde{z}_{t}}{w}\right)$, then $\widetilde{z_{t}}$ is random projection of $z_t$ and $\widetilde{z_{t}} = z_t \cdot w$.
\end{proof}

While have a model, 
efficient training is still unresolved. Next, we show how for \eqref{eq:me_ode}, we can derive ELBO-like bounds using Approximate Bayesian Computation (ABC) \citep{wilkinson2013approximate, fearnhead2010semi}.

\paragraph{Connection with approximation of Stratonovich's SDE.}

Observe that if we select $f$ and $g$ in  \eqref{eq:strat_approx} as
$f(z, t)=\Gamma(z, t) \beta$
and
$g(z, t)=\Gamma(z, t) \Sigma_b^{1/2}$, the approximation of Stratonovich's SDE becomes ME-ODE defined in \eqref{eq:lat_me}, i.e., we set $\dot{z}_{t}=\Gamma(z, t)\mathbf{w}$,
 where $\mathbf{w} = \beta + b \Sigma_{b}^{1/2} \sim \mathcal N(\beta, \Sigma_{b})$.
 
{ Note that our ODE/SDE based derivation of the expression in \eqref{eq:lat_me} coincides with the mixed effects form proposed in \cite{xiong2019training} for single panel (time) data.}
 
\begin{remark} 
Recall that in \eqref{eq:z_h}, we assumed that $h$ is  smooth. In theory, the smoothness assumption is justified due to Wong-Zakai approximation, see  \eqref{eq:strat_approx}. In practice, this can be achieved by choosing a sufficiently fine discretization. %
\end{remark}

\subsection{Model training: ME-NODE ELBO} %

The objective of the training scheme we describe now is to learn \begin{inparaenum}[\bfseries (a)] \item distribution of $z_0$, \item fixed effect $\boldsymbol{\beta}$, \item variance of random effect $\Sigma_b$ \end{inparaenum}. 
To reduce clutter, in this section, we drop index $i$ (which specifies a  subject).

At a high level, our approach is to infer the random effects $\mathbf{b}$ by learning it as a parameter. For our purposes, learning $\mathbf{b}$ corresponds to ensuring that $\mathbf{b}$ satisfies the following key requirement (accounting for a small reconstruction error): $\mathbf{b}$ needs to be random by design for statistical reasons such as uncertainty quantification. 
Using our model in \eqref{eq:me_ode}, it is easy to see that this requirement is satisfied because  $\mathbf{b}$ is sampled from $\mathcal{N}(\mathbf{0}, \Sigma_b)$. 
A common strategy to satisfy the requirement is to use a VAE \citep{chen2016variational}. It is known that in such probabilistic models computing the marginal likelihood $p(x)$ is usually intractable. Let $p(\z, \mathbf{w})$ be the prior joint distribution, $q(\z, \mathbf{w})$ as approximate joint posterior, and $p(x\mid \z, \mathbf{w})$ as likelihood of reconstruction. Using concepts from \S\ref{sec:background}, we can derive a lower bound for the $p(x)$ of our ME-NODE model as:

\vspace*{-2em}
\begin{align}
\log p(x)&=%
\log \int p(x | \z, \mathbf{w}) p(\z, \mathbf{w}) \frac{q(\z, \mathbf{w})}{q(\z, \mathbf{w})} d(\z, \mathbf{w})\label{eq:marg_prop}\\
&=\log \mathbb{E}_{q(\z, \mathbf{w})}\left(p(x \mid \z, \mathbf{w}) \cdot \frac{p(\z, \mathbf{w})}{q(\z, \mathbf{w})}\right)\nonumber\\
\geq \mathbb{E}_{q(\z, \mathbf{w})}&\log p\left(x | \z, \mathbf{w}\right)-KL(q(\z, \mathbf{w}) \| p(\z, \mathbf{w})),\label{eq:elbo_me_prelim}
\end{align}
where  \eqref{eq:marg_prop} follows from the marginalization property and then we use Jensen's inequality.
Next, we define $q(\z, \mathbf{w})$ which is needed to compute ELBO.
Note that 
in the following description, $\xi|\psi$ refers to the random variable $\xi$ conditioned on a value of $\psi$, regardless of what the value is, i.e., it can be $\xi|\psi=0$ or $\xi|\psi=1$.

\paragraph{Defining $q(\z, \mathbf{w})$.}
Assuming that $z_0$ and $\mathbf{w}$ are independent random variables, we get 
$$
q(\z, \mathbf{w})=q\left(z_{0}, \z_{-0}, \mathbf{w}\right)=q\left(\z_{-0} \mid z_{0}, \mathbf{w}\right) q(z_0)q(\mathbf{w}).
$$
Recall that $\z_{-0}=(z_1,\ldots,z_n)$ is a vector of ODE solutions at time step $t$, except $t=0$. 
At every step $t$, $z_t$ is a random variable, which is a function of $z_0$ and $\mathbf{w}$. 
However, with a fixed initial point $z_0$ and mixed effect $\mathbf{w}$, the progression follows a {\em defined} trajectory (i.e., there is no randomness). It means that $z_t|z_0, \mathbf{w}$ is deterministic and hence the distribution  $q(z_{t}|z_0, \mathbf{w})$ is degenerate \citep{danielsson1994stochastic}, which results in
$q(\z_{-0}|z_0, \mathbf{w}) = \mathds{1}_{\obs{\z}_{-0}}\{\z_{-0} | z_{0}, \mathbf{w}\}$. Thus,
\begin{equation}
q(\z, \mathbf{w})=\mathds{1}_{\obs{\z}_{-0}}\{\z_{-0} | z_0, \mathbf{w}\} q\left(z_{0}\right) q(\mathbf{w})
\label{eq:q_par}
\end{equation}

{\bf Note.} While the derivation from \eqref{eq:marg_prop} to \eqref{eq:elbo_me_prelim} is well defined for point mass distributions stated in \eqref{eq:q_par}, the use of `KL',
although consistent with  \cite{bai2020efficient} 
(pp 3, (4)--(6)) 
is not ideal (because $\log(0)$ and thus $KL$ is undefined). We 
will avoid using $KL$ notation in the loss in \eqref{eq:elbo_me}.

\paragraph{MC approximation of $E_{q(\z, \mathbf{w})}g(\z, \mathbf{w})$.}
The key in computing  \eqref{eq:elbo_me_prelim} is to estimate $E_{q(\z, \mathbf{w})}g(\z, \mathbf{w})$ for a given function $g(\z, \mathbf{w})$.
Based on parameterization of $q(z, \mathbf{w})$ in \eqref{eq:q_par},  
\begin{equation*}
\begin{split}
    E_{q}&g(\z, \mathbf{w}) =\\ &\int_{z_0, \mathbf{w}}g(\z, \mathbf{w}) \cdot
    \mathds{1}_{\obs{\z}_{-0}}\{\z_{-0} | z_{0}, \mathbf{w}\}q\left(z_{0}\right) q(\mathbf{w}) dz_0 d\mathbf{w}.
\end{split}
\end{equation*}

While the integration may  be intractable, it can be estimated by Monte Carlo (MC) techniques. Sampling ($z_0^m$, $\mathbf{w}^m$) from q, we compute $\frac{1}{M} \sum_{m=1}^M g^*(\z^m, \mathbf{w}^m)$, where $$g^*(\z^m, \mathbf{w}^m) = g(\z^m, \mathbf{w}^m)\cdot \1[\obs{\z}_{-0}]{\z_{-0} | z_{0}^m, \mathbf{w}^m}.$$ This type of sampling is called likelihood-free rejection sampling \citep{del2012adaptive}: we reject all samples ($z_0^m$ and $\mathbf{w}^m$), which do not generate observed  $\obs{\z}_{-0}$.

\paragraph{The final loss.}
Given this MC approximation (with $M$ samples), with the approximate posterior $q(\z, \mathbf{w})$ defined in \eqref{eq:q_par} and with a similarly defined prior
$p(\z, \mathbf{w}) = \mathds{1}_{\obs{\z}_{-0}}\{\z_{-0} | z_0, \mathbf{w}\} p\left(z_{0}\right)p(\mathbf{w})$,
the final loss is 
\begin{equation}
\begin{split}
&\frac{1}{|S|} \sum_{s\in S} \left(\log p\left(x | \z^s, \mathbf{w}^s\right) - \log\frac{q(z_0^s)q(\mathbf{w}^s)}{p(z_0^s)p(\mathbf{w^s})}\right),\\
\end{split}
\label{eq:elbo_me} %
\end{equation}
where $S$ is a set:  $\{\forall s \in S: \1[\obs{\z}_{-0}]{\z_{-0} | z_{0}^s,\mathbf{w}^s} = 1\}$. 

\begin{remark}
While a MC approximation in \eqref{eq:elbo_me} is an unbiased estimator of the Lower Bound in \eqref{eq:elbo_me_prelim}, its variance is $O\left(\frac{1}{|S|}\right)$. This leads to efficiency issues in that it may require a large $M$ until we get $z_0$ and 
$\mathbf{w}$ to generate $\z_{-0}$ exactly along the observed trajectory to 
populate the set $S$. But we can address this problem using ABC methods \citep{wilkinson2013approximate, fearnhead2010semi}.
\end{remark}

\paragraph{Efficient sampling: approximating $\1[\z_{-0}^{\text{obs}}]{\z_{-0} | z_{0}, \mathbf{w}}$.}
ABC recommends finding samples of $z_0$ and $\mathbf{w}$ to generate trajectories $\z_{-0}$ which are {\em approximately} equal to the observed one, rather than exactly equal. The idea in \citep{marin2012approximate} proposes using \1[y]{z} as  $\1[A_{\epsilon, y}]{z}$, where 
$A_{\epsilon, y}=\{z \mid d\{z, y\} \leq \epsilon\}$ is an $\epsilon$-neighborhood of $y$, and $d$ is a distance function. 
For a direct application of these methods on $\1[\z_{-0}^{\text{obs}}]{\z_{-0} | z_{0}, \mathbf{w}}$, we must have access to $\obs{\z}_{-0}$ in the latent space, which is unavailable unless the encoder $E$ and the decoder $D$ are identity functions. 
Nonetheless, we can approximate $\1[\z_{-0}^{\text{obs}}]{\z_{-0} | z_{0}, \mathbf{w}}$, by comparing if the decoded $\z_{-0} | z_{0}, \mathbf{w}$ indeed corresponds to $\obs{\boldsymbol{x}}_{-0}$, i.e., we need to compute $\1[\obs{\boldsymbol{x}}_{-0}]{D(\z_{-0} | z_{0}, \mathbf{w})}$.
We simply use the mean squared error (MSE) as the distance for this comparison.

By decreasing $\varepsilon$, we can improve the quality of the samples for MC estimation, but at higher compute cost. However, because we learn the distribution of $z_0$, during the first steps of training, our model provides poor reconstructions. For this reason, setting $\varepsilon$ to a small value at the beginning of the training is inefficient.
Therefore, we make $\varepsilon$ adaptive through the training, by choosing the sample, closest to our observed trajectory, i.e., sample with smallest distance $d$, 
and $\varepsilon$ is a function of the initial point $z_0$.

{\bf Choice of $q(\textbf{w})$.} 
To optimize the ELBO, it is necessary to define the approximate posterior $q(\mathbf{w})$ and prior $p(\mathbf{w})$. While we assumed in \eqref{eq:lat_me} and Lemma~\ref{lemma:randproj} that the true distribution of $\mathbf{w}$ is Normal, $q(\textbf{w})$ and $p(\textbf{w})$ remain design choices for the user. For example, if we believe that the correlation structure of the data is sparse, then we have the following choices: Horseshoe \citep{carvalho2009handling}, spike-and-slab with Laplacian spike \citep{deng2019adaptive} or Dirac spike \citep{bai2020efficient}. However, for our experiments, we found that modeling $q(\textbf{w})$ and $p(\textbf{w})$ as Normal is sufficient.

{\bf Calibration.}
One feature of our model is that learned  distribution of mixed effects $\mathbf{w}$ can be used for personalized prediction during extrapolation \citep{wang2014estimating,ditlevsen2005mixed,bouriaud2019comparing}.
First, we train our model to learn the parameters of distribution of $z_0$, fixed effects $\boldsymbol{\beta}$, and variance of random effect $\Sigma_b$.
Then at test time, we make use of the observed temporal data for a previously unseen test subject $\obs{\boldsymbol{x}}$. %
Given the learned distribution of mixed effects $\mathbf{w}$, we 
want to find a sample $w \sim \mathbf{w}$, which minimizes error w.r.t. 
$\obs{\boldsymbol{x}}$. 
This 
selection provides the most appropriate 
mixed effect $w$ corresponding to 
$\obs{\boldsymbol{x}}$. We call this process calibration: a solution to  $\arg\!\min_{w \sim \mathbf{w}} \text{MSE}\left(\obs{\boldsymbol{x}}, \boldsymbol{\hat{x}}(w)\right)$, where $\boldsymbol{\hat{x}}(w)$ is a prediction from our model. 
Note that this is slightly different from the average (used in probabilistic models like VAE).

{\bf Method summary.} We provide a step-by-step summary, 

\begin{mybox}{gray}{\bf Train and test phases}
In the {\bf training} phase, 
the observed data for subject $i$, $x^i_t$ for $t \in [0, T]$ is assumed to be provided. Then, 
we 
\begin{compactenum}
\item Use a suitable encoder $E$ to map $\boldsymbol{x}^i=\{x^i_t\}$ to the latent representation of initial point $z_0$ of underlying ODE, as $z_{0}^{i} \sim \mathcal{N}(\mu, \sigma)$, where $\mu, \sigma = E(\boldsymbol{x}^i)$.
\item Given a suitable decoder $D$, we fit 
the ME-NODE model, by minimizing the loss in \eqref{eq:elbo_me}, thereby learning the appropriate distribution of mixed effects $\mathbf{w^i}$.
\end{compactenum}
The output from this phase 
latent representation $z^i_t$ described by ME-NODE model and 
the corresponding distribution of mixed effects $\mathbf{w^i}$.

In the {\bf test} phase, the 
observed data for subject $i$, $x^i_t$ for $t \in [0, T]$ is assumed to be provided. Then, we
\begin{compactenum}
\item Select a personalized mixed effect $w^i \sim \mathbf{w^i}$, according to the calibration scheme. 
\item Then, we use the selected mixed effect sample $w^i$, to generate personalized prediction for the subject $i$ for either interpolation or extrapolation.
\end{compactenum}
The output from 
this phase is the prediction for subject $i$, $\hat{x}^i_t$ for $t \in [0,\ldots, T^*]$, where $T^*$ can be extrapolated time, i.e., $T^* \ge T_i$, and/or a denser interpolation in  $[0, T]$.
\end{mybox}

\section{Experiments}
We evaluate our model on five temporal datasets: \begin{inparaenum}[\bfseries (1)] \item simulations, \item MuJoCo hopper, \item rotating MNIST, and \item two different Neuroimaging datasets, representing disease progression in the brain. 
\end{inparaenum} 

{\bf Goals.} We will evaluate: \begin{inparaenum}[\bfseries (a)] 
\item  ability to learn mixed effects, given different types of correlations in the data
\item the effect of mixed effects dimension $m$ on extrapolation power and confidence of the model, and
\item the ability to preserve statistical group differences in the data in latent representations.
\end{inparaenum}

The baselines are given separately for each experiment. We provide description of hardware and neural networks architectures, including encoder/decoder in appendix.

\vspace*{-1em}
\subsection{Synthetic dataset}
We start with a synthetic setup where all parameters are known. Using a ODE solver and conditioning on $z_0$ and $w$, we generate a solution of the mixed effect ODE 
\begin{equation}
\left[
\begin{array}{l}
 z_{0}^{i} \sim \mathcal{N}(\mu = 1.3, \sigma = 0.01)\\
 w^i \sim \mathcal{N}(\boldsymbol{\beta} = 0.3, \sigma_b = 0.01)\\
 \dot{z}^i_t = z^i_t w^i\\
\end{array}
\right.
\label{eq:toy_data}
\end{equation}
We set the encoder $E$ and decoder $D$ in 
\eqref{eq:me_ode} to the identity transformation.
Given $1000$ ($80:20$ split for train/test) numerical solutions of the ODE in \eqref{eq:toy_data}, we 
uniformly sampled $20$ time points from $[0,3]$, and use 
the first $10$ time steps for interpolation and the last $10$ for extrapolation. %

{\bf Parameters.} As the optimization of ELBO in  \eqref{eq:elbo_me} requires samples, we evaluate the performance of our model by varying the number of samples for $z_0$ and $w$ (denoted by $n_{z_0}$ and $n_w$ respectively). The results in Table \ref{tab:toy_data} suggest that MSE goes down with an increase of $n_{z_0}$ or $n_w$.

In Figure~\ref{fig:toy_extrap} ({\it top panel}), we show samples (trajectories) drawn from the learned model (blue lines) with the real trajectories (`x' marker). Notice that the sampled trajectories from the learned model almost cover the ``range''  of real trajectories and the results appear meaningful.

\begin{table}[!tb]
\centering
\scalebox{0.88}{
\begin{tabular}{ c |c c c c | c}
\specialrule{1pt}{1pt}{0pt}
\rowcolor{tableheadcolor}
{Estimated}& \multicolumn{4}{c|}{$n_{z_0}, n_w$} &
{True}\\
         \hhline{>{\arrayrulecolor{bluecolor}}->{\arrayrulecolor{black}}---->{\arrayrulecolor{bluecolor}}->{\arrayrulecolor{black}}} %
          \rowcolor{tableheadcolor}parameters  & $1, 1$ & $1, 10$ & $10, 1$ & $10, 10$    & values \\
 \toprule
 $\hat{\mu}$ & $1.252$ & $1.258$ & $1.316$ & $1.313$ & $1.3$\\
 $\hat{\sigma}$ & $0.003$ & $0.002$ & $0.005$ & $0.016$ & $0.01$\\
 $\hat{\boldsymbol{\beta}}$ & $0.311$ & $0.315$ & $0.311$& $0.319$ & $0.3$ \\
 $\hat{\sigma_b}$ & $0.036$ & $0.051$ & $0.054$ & $0.060$ & $0.01$\\
 \bottomrule
 MSE (all) & 0.0017 & 0.0011 & 0.0006 & 0.0005 & \cellcolor{gray}  \\
\end{tabular}
}
\caption{\footnotesize The first four rows show the estimated parameters for specific choices of $n_{z_0}$ and $n_w$. Here we use the following convention: $(n_{z_0}, n_w)=(i, j)$ denotes we draw $i$ and $ij$ number of samples from $z_0$ and $w$ respectively. The last row presents the MSE values for the estimated parameters.} 
\label{tab:toy_data}
\vspace*{-1em}
\end{table}

{\bf Mixed effects.} To evaluate  generation of a personalized prediction  for a subject $i$, by learning mixed effect $w^i$, recall that we split our data in two parts: interpolation (observed) and extrapolation (unknown). We use observed samples (interpolation part) to calibrate the mixed effect $w^i$ and pass $w^i$  to the selected trajectory during extrapolation. Figure~\ref{fig:toy_extrap} ({\it bottom panel}), shows that personalized prediction (green line) follows the observed data nicely across the entire time interval. In comparison, the standard BNN approach generates trajectory close to observed data for interpolation, and fails to extrapolate as well as our proposed model.

{\bf Runtime.} For 1000 samples, runtime for an epoch of our method is 2.1 seconds, while Neural SDE takes about 27.13 seconds for the same memory utilization ($\sim$965MB).

\begin{figure}[!t]
    \centering
    \includegraphics[width=0.49\columnwidth, trim={0.8cm, 0.8cm, 0.2cm, 0.3cm}, clip]{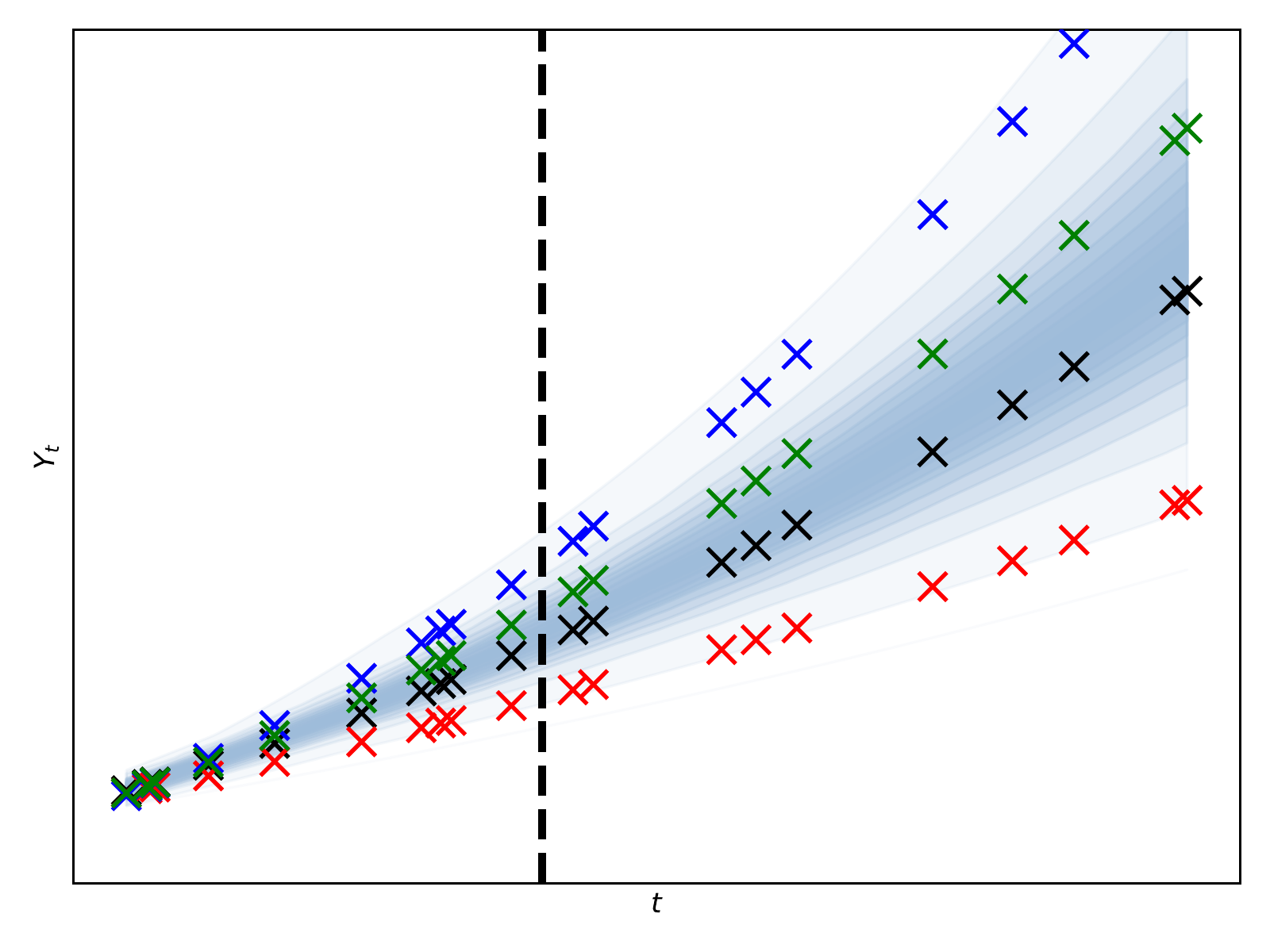}
    \includegraphics[width=0.49\columnwidth, trim={0.8cm, 0.8cm, 0.2cm, 0.3cm}, clip]{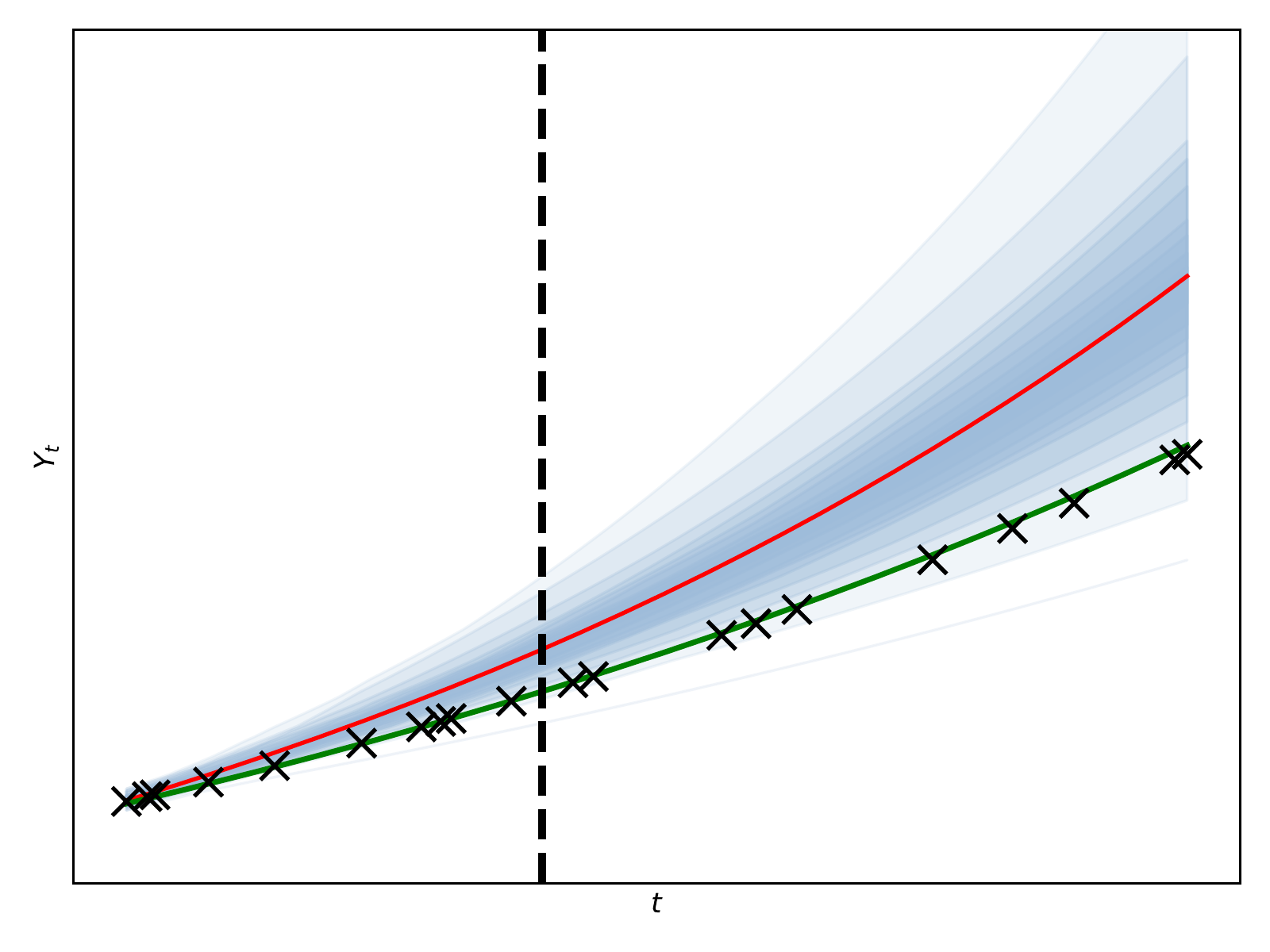}
    
    \caption{\footnotesize \footnotesize Real samples are denoted by ``x" while blue lines are predictions. For each subfigure, the LHS and RHS of the dotted line contains interpolation and extrapolation results respectively.  %
    \textit{Left:} the learned distribution of trajectories (from all test subjects) satisfies dynamic process described by the known system \eqref{tab:toy_data}. \textit{Right:} using the interpolation for calibration, we infer mixed effect $w^i$ for $i^{th}$ trajectory and generate a personalized prediction (green line). The prediction using BNN is  shown in red.}
    \label{fig:toy_extrap}
    \vspace*{-2em}
\end{figure}

\vspace*{-1em}
\subsection{MuJoCo Hopper}
Here, we evaluate the performance of our model for simple Newtonian physics.
Similar to  NODE \citep{rubanova2019latent}, we created a physical simulation using  MuJoCo Hopper. While in \citep{rubanova2019latent}, the generated samples were i.i.d, we explicitly introduce correlation  between the samples. The process of MuJoCo Hopper is defined by the initial position and velocity. In order to generate correlated samples, we choose the initial velocity from the pre-specified set containing $1, 4$ or $8$ vectors. The entries of the velocity vectors are uniformly sampled from $[-2,2]$.
We evaluate our model on interpolation (10 steps) and extrapolation (10 steps) and compare results with NODE in  Table~\ref{tab:mujoco_data}. As our model implicitly learns correlation structure of the data by learning the distribution of mixed effect $q(\mathbf{w})$, we see an improvement in both interpolation and extrapolation.
In addition, Figure~\ref{fig:mujoco} presents  representative extrapolated samples using our proposed model.

\begin{table}[!b]
\centering
\scalebox{0.95}{\footnotesize
\begin{tabular}{  c c c c c c  }
\cmidrule[\heavyrulewidth]{1-5}
\addlinespace[-\belowrulesep]
  {\cellcolor{tableheadcolor}} &
  {\cellcolor{tableheadcolor}}& \multicolumn{3}{c}{\cellcolor{tableheadcolor}velocities}\\
  \hhline{-----}
\rowcolor{tableheadcolor}
  & model & 1 & 4 & 8\\
  \toprule
\multirow{2}{*}{Interpolation} & NODE & $7.4$ & $5.4$ & $5.5$   \\
                    & This work& \cellcolor{Maroon}${5.7}$ & \cellcolor{Maroon}${4.6}$ & \cellcolor{Maroon}${4.6}$ \\   
                              \midrule
\multirow{2}{*}{Extrapolation} & NODE & $166.1$  & $82.1$ & $80.3$   \\
                    & This work& \cellcolor{Maroon}${164.1}$  & \cellcolor{Maroon}${81.2}$ & \cellcolor{Maroon}${80.0}$ \\ 
 \bottomrule
\end{tabular}
}
\caption{\footnotesize MSE (in scale of $10^{-3}$) on MuJoCo Hopper data set, generated for three settings: 1, 4 and 8 initial velocities. We compare these two models using identical neural networks with the same number of levels and hyperparameters, however, in our model the dimension of mixed effect is $m=50$.}
\label{tab:mujoco_data}
\end{table}

\begin{figure}[!t]
    \centering
    
    \includegraphics[width=0.95\columnwidth, trim={31.65cm, 0cm, 0cm, 0cm}, clip]{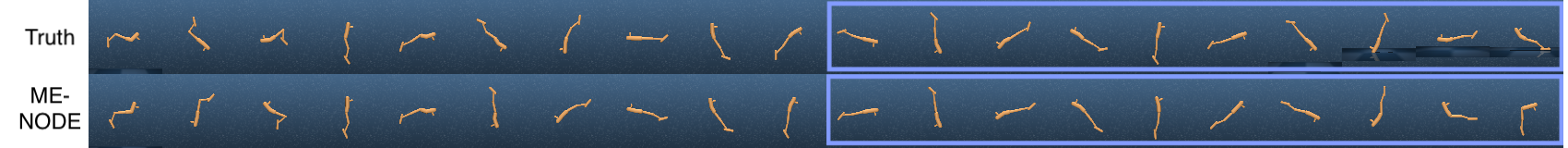}
    \caption{\footnotesize Visualization of $10$ steps of extrapolation after observing $10$ previous steps, with dimension of mixed effect $m=50$. 
    \textit{Top} ground truth, \textit{bottom} our prediction.}
    \label{fig:mujoco}
\end{figure}
\vspace*{-1em}

\subsection{Rotating MNIST}
We now 
evaluate a slightly more complicated rotating MNIST dataset. %
Here, we consider different types of correlations in the data and check: \begin{inparaenum}[\bfseries (i)] \item relation between mixed effect dimension $m$ (which we can think of as a dimension of random projection) and performance of the model, \item  the performance of personalized prediction for extrapolation in comparison with  standard BNN. \end{inparaenum}

{\bf Data description.} 
Similar to the setup in ODE2VAE \citep{yildiz2019ode2vae}, we construct a dataset by rotating the images of different handwritten digits, in order to learn a digit specific mixed effect model.
In ODE2VAE, digits were rotated by $22.5^{\circ}$. We used a slightly different scheme: for a sampled digit we randomly choose an angle from the set of $1, 4,$ or $8$ angles from the range $[-\pi/4, \pi/4]$ and apply it at all time steps. 
For example, if we choose the set containing $4$ angles, then a sampled digit is rotated using  one of the $4$ angles, selected randomly. 
In order to simulate a practical scenario, %
we spread out the initial points, by randomly rotating  a digit by angles from $-\pi/2$ to $\pi/2$. 
We generate 10K samples of different rotating digits for $20$ time steps and split it in  two equal sets: interpolation and extrapolation. %

{\bf Effect of mixed effect (random projection) dimension $m$.}
Recall from Lemma \ref{lemma:randproj}, we showed  that mixed effects $\mathbf{w}$ in \eqref{eq:me_ode} can be considered as a random projection. So, we can expect that with an increase in $m$, MSE should decrease as it leads to a richer latent  representation of a trajectory. We see that this is indeed true as shown in Figure~\ref{fig:mnist_bar} ({\it left panel}). 
Here, we demonstrate the MSE of reconstruction for three values of $m$: $1, 20,$ and $50$.  We observe that for each time step, MSE decreases monotonically with an increase of $m$. 

\begin{figure}[!bt]
    \centering
    \scalebox{0.99}{
    \includegraphics[width=0.55\columnwidth]{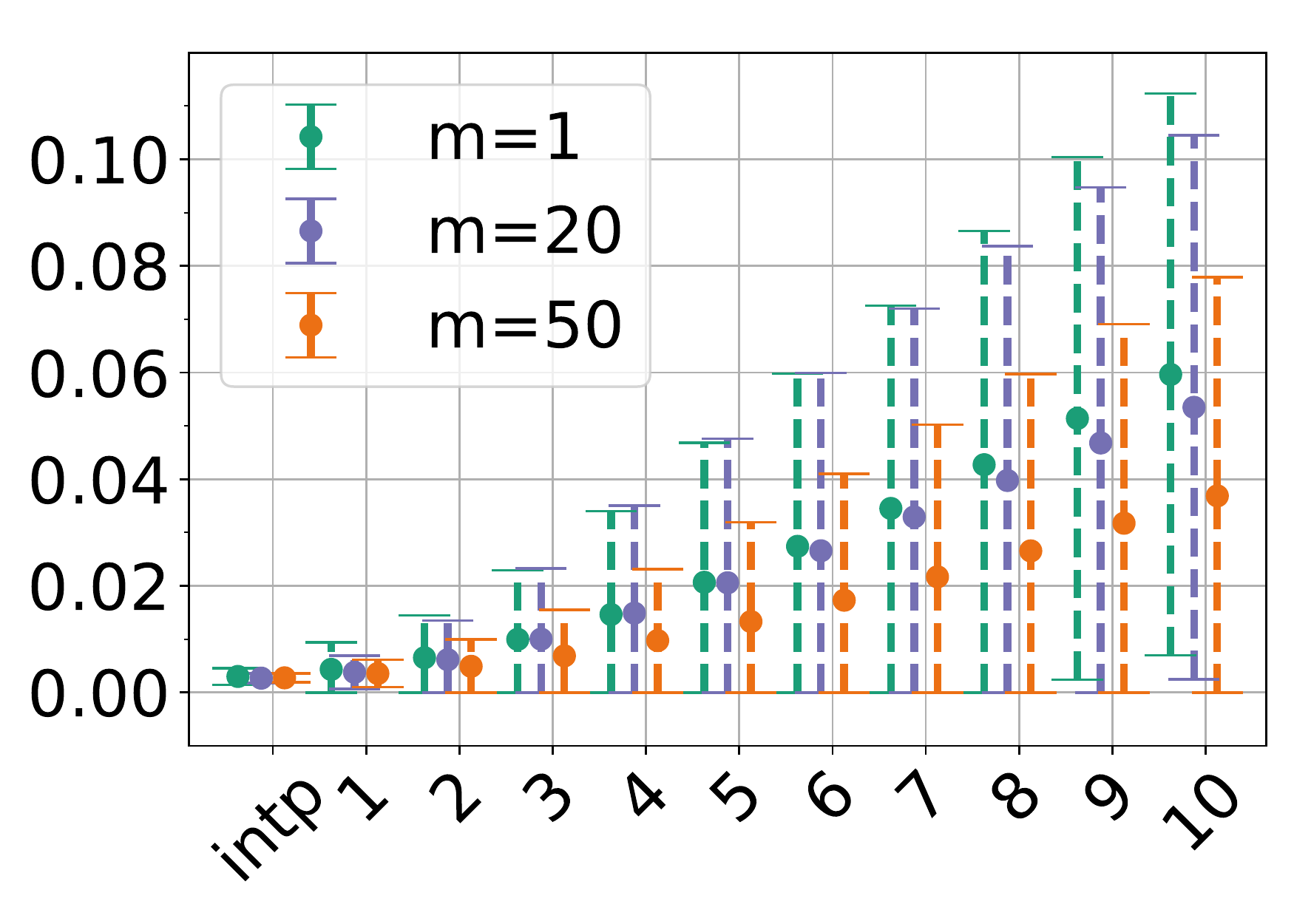}
    \includegraphics[width=0.465\columnwidth,trim={2.5cm, 0.0cm, 0.0cm, 0.0cm},clip]{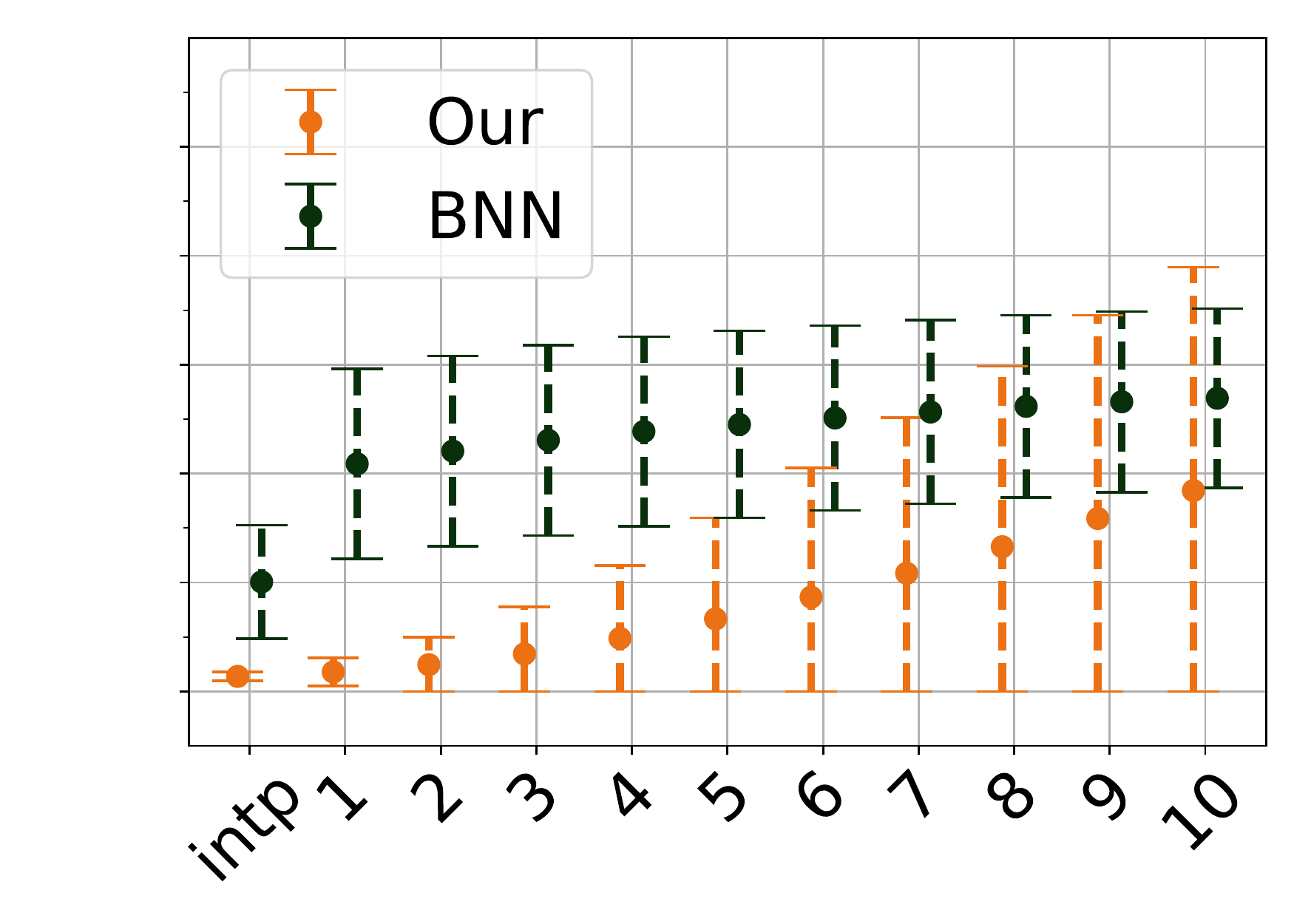}
    }
    \vspace*{-12pt}
    \caption{\footnotesize Distribution ($\mu\pm\sigma$) of MSE:
    \textit{left:}  varying with $m$: $1, 20, 50$,
    \textit{right:} 
    first point on x-axis (`intp') indicates  average MSE for all time steps: (1-10) of interpolation, and 1-10 indicate steps of extrapolation. Results for data with $8$ possible angles shown.
    }
    \label{fig:mnist_bar}
\end{figure}

{\bf Calibration for personalized prediction.}
Here we choose $m=50$ and compare our model for interpolation with ODE2VAE using $3$ different rotation settings: $1, 4,$ and $8$ available angles. The calibration results in Table~\ref{tab:ode2vae_our} show that our model significantly outperforms the baseline ODE2VAE; however, making the correlation structure of the data more complicated (increasing number of possible rotation angles), does lead to a larger MSE. This is expected: with an increase in complexity of correlation in data, the learning task  (and thereby, prediction) becomes harder.

Recall that in order to generate a personalized prediction for subject $i$ we have to sample mixed effect $\mathbf{w}^i$ resulting in trajectory closed to the observed. Thus, if the model fails to learn the distribution $q(\mathbf{w})$ accurately, sampling such $\mathbf{w}^i$  is less likely, and will result in a larger interpolation error. %

For extrapolation, we compare with a standard BNN approach  in Fig.~\ref{fig:mnist_bar}. We observe that for each extrapolation step, we obtain, on average, much smaller MSE and smaller variance during the initial steps of extrapolation.
\vspace*{-1em}
\begin{table}[!b]
\vspace*{-1em}
\centering
\scalebox{0.95}{
\begin{tabular}{  c c c c c c  }
\cmidrule[\heavyrulewidth]{1-5}
\addlinespace[-\belowrulesep]
  {\cellcolor{tableheadcolor}}& 
  {\cellcolor{tableheadcolor}}& \multicolumn{3}{c}{\cellcolor{tableheadcolor}angles}\\
  \hhline{-----}
\rowcolor{tableheadcolor}
  & model & 1 & 4 & 8\\
  \toprule
\multirow{2}{*}{Interpolation} & ODE2VAE & $0.0648$ & $0.0644$  & $0.0640$   \\
                    & Ours-50& \cellcolor{Maroon}${0.0006}$ & \cellcolor{Maroon}${0.0014}$ & \cellcolor{Maroon}${0.0027}$ \\   
                             \bottomrule%
                    
\end{tabular}
}
\caption{\footnotesize MSE of two models, given different complexity of the data. Low interpolation error of our model indicates the properly learned mixed effect distribution $q(\mathbf{w})$.}
\label{tab:ode2vae_our}
\vspace*{-2em}
\end{table}
\paragraph{Extrapolation steps.} Earlier, the number of steps for extrapolation were smaller than the number of observed steps, used for calibration. In Fig. ~\ref{fig:calibrat}, we show results of interpolation and extrapolation, varying the number of observed time steps used for calibration. Expectedly, decreasing the number of steps to be small for calibration yields a smaller number of steps where the extrapolation is meaningful. 

\begin{figure}[!tb]
    \centering
    \includegraphics[width=0.98\columnwidth, trim={1.9cm, 9.5cm, 2.5cm, 10cm}, clip]{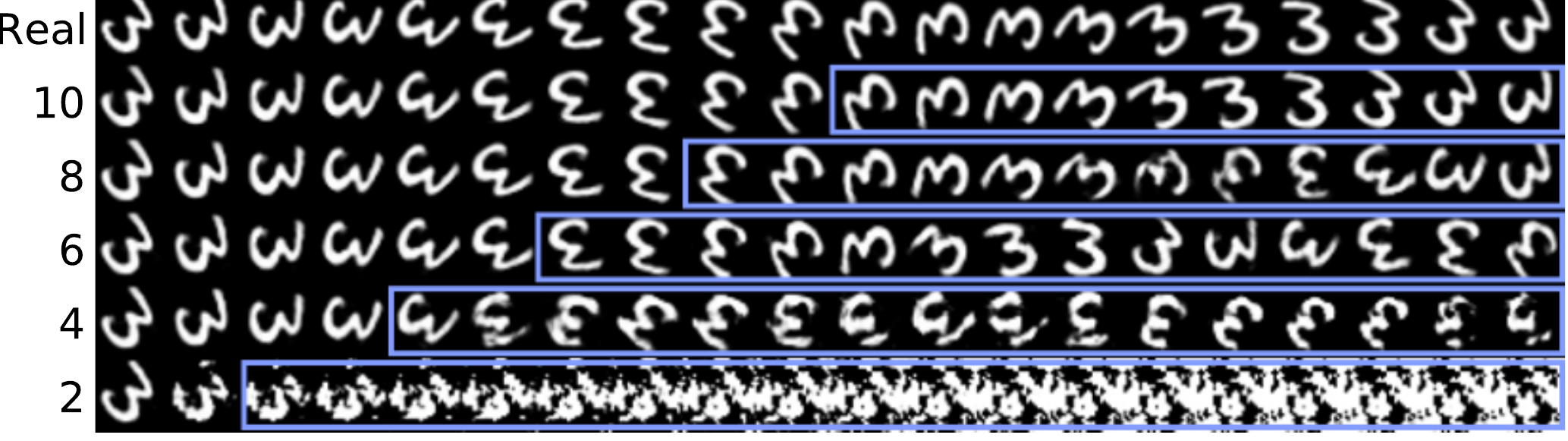}
    \caption{\footnotesize Visualization of extrapolated results (blue frame), given $n$ time steps for calibration, where $n$ is shown on y axis. The top row indicate real data.}
    \label{fig:calibrat}
    \vspace*{-1em}
\end{figure}
\vspace*{-1em}
\subsection{Longitudinal Neuroimaging data}
In this section, we conduct experiments on two longitudinal brain imaging datasets obtained 
from Alzheimer's Disease Neuroimaging Initiative (ADNI) (\url{adni.loni.usc.edu}), both of which describe AD progression through time, but are derived from 
two different imaging modalities. 

\paragraph{Effect of number of dimensions $m$.}
We conducted experiments for different values of the mixed effect (random projection) dimension $m$, see appendix. 
We find that while for any $m$ interpolation looks similar to real data, 
the further we move in extrapolation, the more noticeable the differences are. For example, for $m=1$ some frames look blurry and in the last steps of extrapolation, 
the rotation is wrong. 
Increasing the dimension of random projection to $m=20$ improves image quality, but does not fix rotation. Increasing dimension further to $m=50$, not only improves quality of digits, but also leads to a better prediction of rotation.

\begin{figure}[!bt]
    \centering
    \scalebox{0.4}{
    \includegraphics[width=0.3\textwidth]{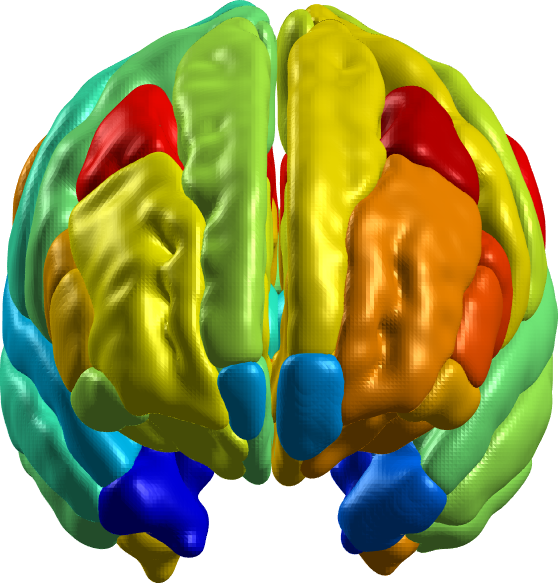}
    \includegraphics[width=0.3\textwidth]{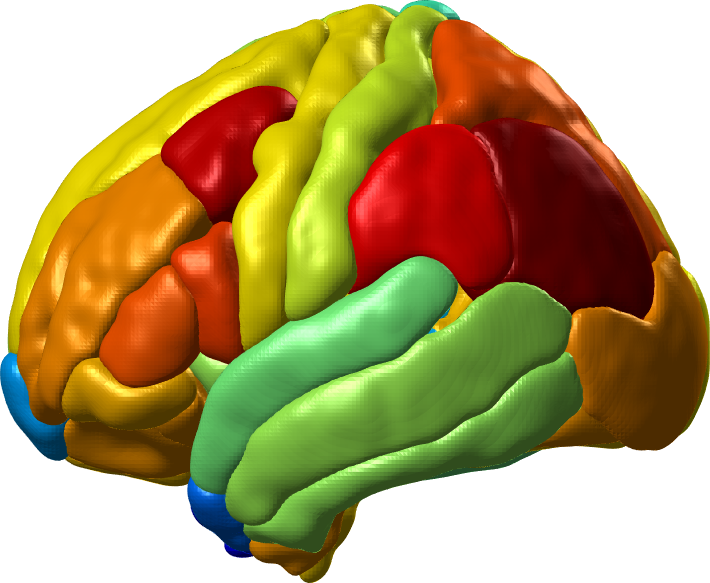}
    \includegraphics[width=0.3\textwidth]{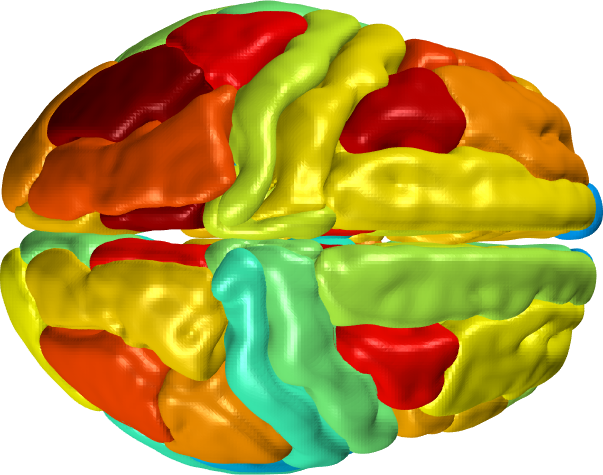}
    }
    
    \scalebox{0.4}{
     \includegraphics[width=0.3\textwidth]{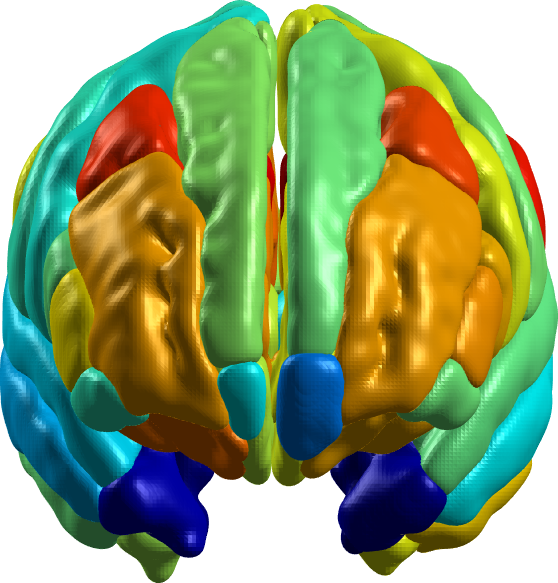}
    \includegraphics[width=0.3\textwidth]{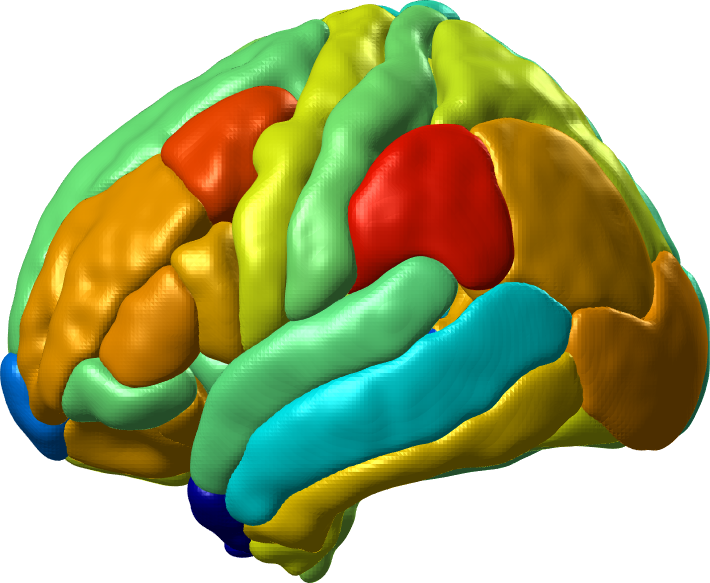}
    \includegraphics[width=0.3\textwidth]{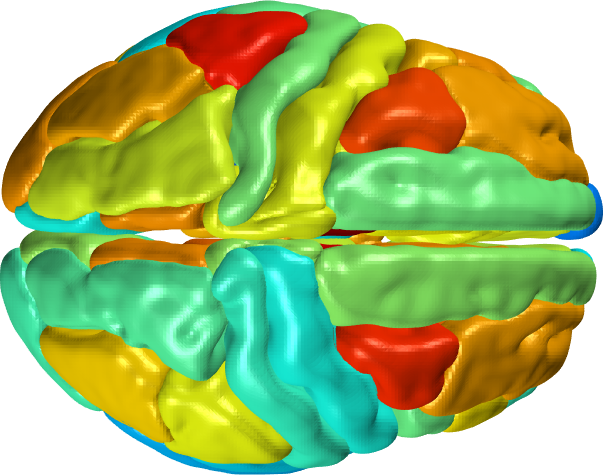}
    }
    
    \scalebox{0.4}{
    \includegraphics[width=0.3\textwidth]{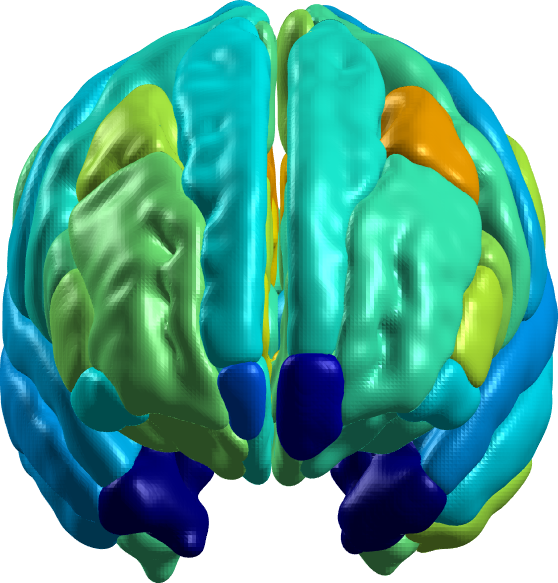}
    \includegraphics[width=0.3\textwidth]{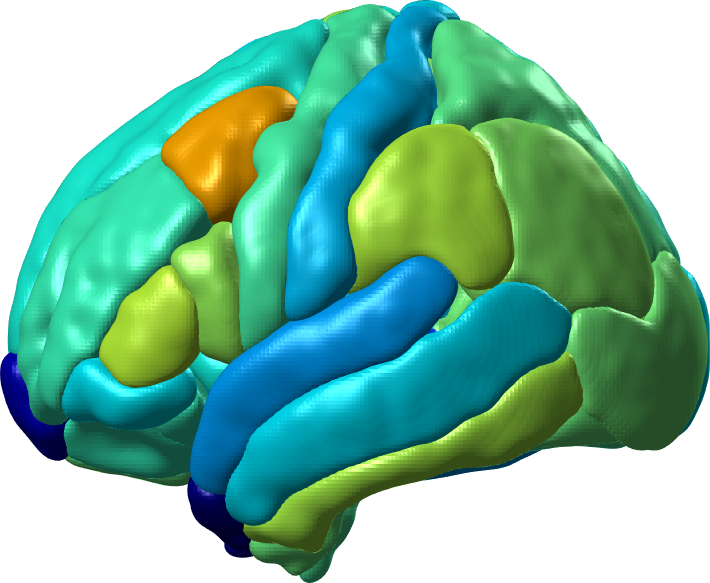}
    \includegraphics[width=0.3\textwidth]{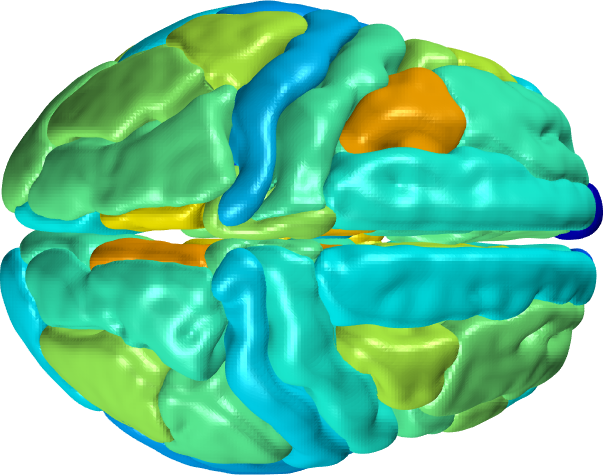}
    }
    
    \includegraphics[width=0.8\columnwidth]{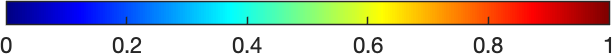}
    \caption{\footnotesize \textit{Top - bottom row:} ground truth, our predicted, BNN predicted. We present the state of the brain at time step $t=3$.  Red and blue indicate high and low AV45 respectively. Compared to the baseline, our model is able to generate a sample for the subject with predicted values of AV45 closer to the ground truth.}
    \label{fig:tadpole}
    \vspace*{-1em}
\end{figure}
{\bf (A) TADPOLE.}
TADPOLE dataset includes data for $276$ participants with $3$ time points. It represents Florbetapir (AV45) Positron Emission Tomography (PET) scans, which measure the level of amyloid-beta pathology in the brain  \citep{marinescu2018tadpole}. Scans were registered to a template (MNI152) to derive the $82$ gray matter regions. Thus, each sample, at time $t$ is a $82$ dimensional vector, i.e., $\mathbf{x}^t\in \mathbf{R}^{82}$.

\begin{figure*}[!bth]
    \centering
    \scalebox{0.6}{
     \includegraphics[width=0.36\textwidth, trim = {1.2cm, 0.6cm, 1.1cm, 1.5cm}, clip]{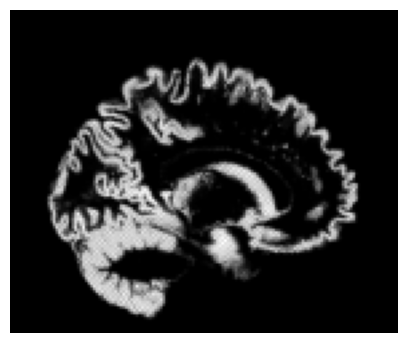}
    \includegraphics[width=0.36\textwidth, trim = {1.2cm, 0.6cm, 1.1cm, 1.5cm}, clip]{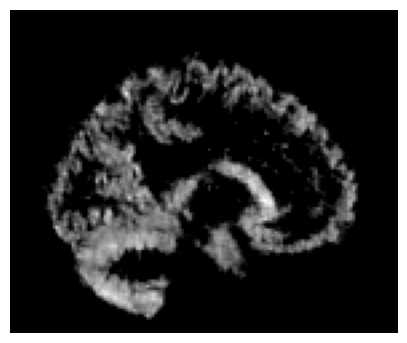}
     \includegraphics[width=0.278\textwidth, trim = {1cm, 1.7cm, 1cm, 1.5cm}, clip]{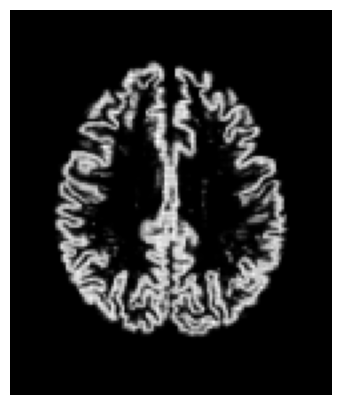}
    \includegraphics[width=0.278\textwidth, trim = {1cm, 1.7cm, 1cm, 1.5cm}, clip]{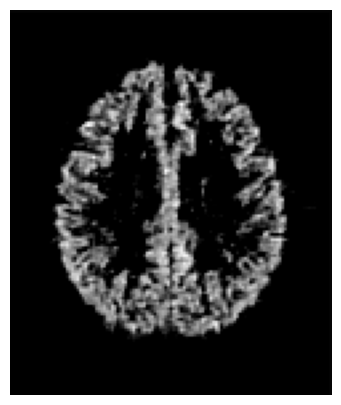}
    }
    \vspace{-8pt}
    \caption{\footnotesize {\it Left-Right} Ground truth (1, 3) and prediction (2, 4) of two slices (side and top) of 3D brain at time point $3$.}
    \label{fig:adni}
\end{figure*}

Given a few time points (3 time points), we evaluate generation of a personalized prediction for a subject in an interpolation setting. We compare our personalized prediction with ground truth and prediction using a standard BNN approach. 
The predictions of both our model and BNN are based on samples from the same distribution $q(\z, \mathbf{w})$. However, Fig. ~\ref{fig:tadpole} shows that the calibration  of our model  (Fig. ~\ref{fig:tadpole}, {\it second row}) provides better prediction than BNN (Fig. ~\ref{fig:tadpole}, {\it third row}). Even though the learned distribution of mixed effects $q(\mathbf{w})$ is capable of providing the correct trajectory (calibrated prediction), the direct application of the model without personalized calibration (BNN) leads to high subject-wise uncertainty. %

{\bf (B) ADNI.}
Our second dataset from ADNI contains 
processed MRIs (3D brain scans) of size $105 \times 127 \times 105$ per subject at $3$ time steps. The subjects are divided into two groups: diagnosed with Alzheimer's disease (abnormal: $377$ subjects) and healthy controls (normal: $152$ subjects).

\begin{figure}[!tb]
    \centering
    \includegraphics[width=0.75\columnwidth, trim={3cm 0.6cm 2.5cm 0.7cm}, clip]{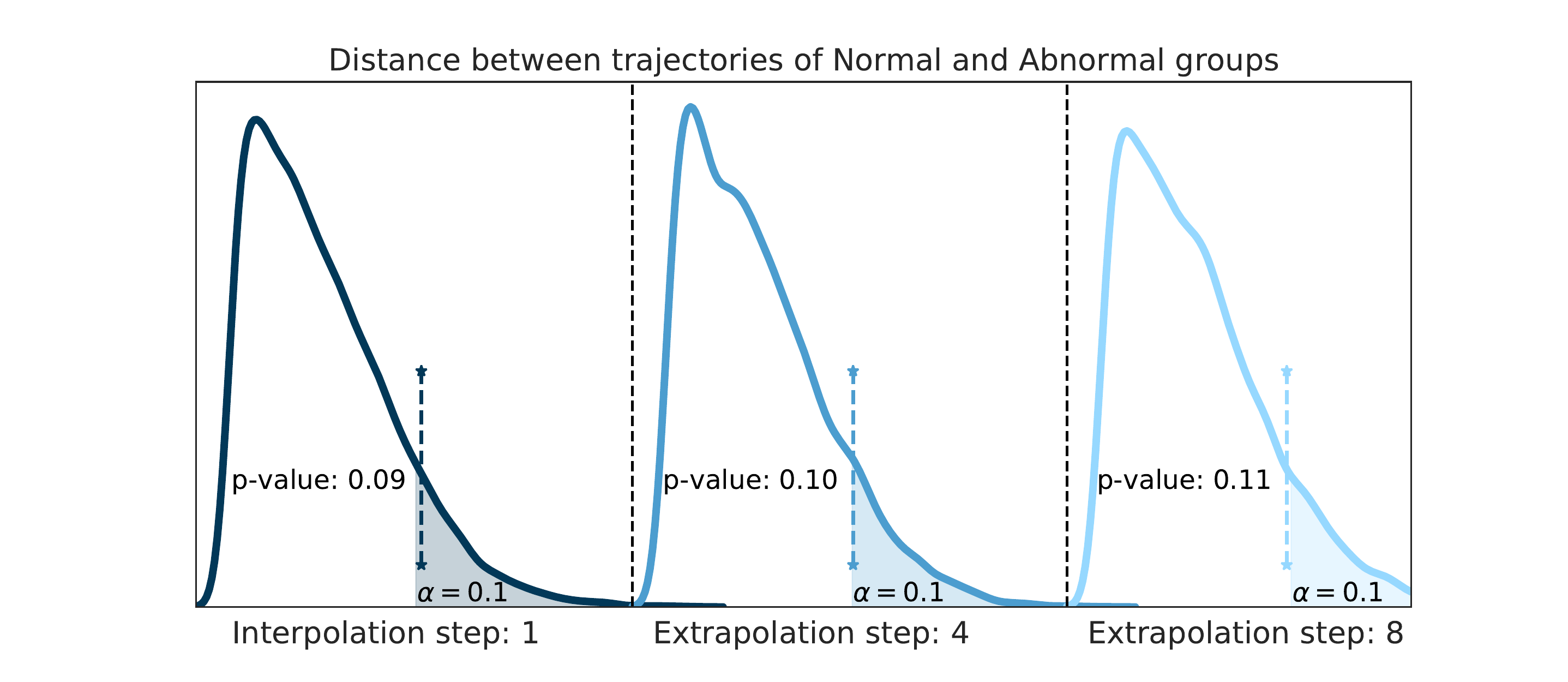}
    \vspace{-8pt}
    \caption{\footnotesize Distribution of distances resulted from permutation test  with original distances (dashed lines), and corresponding $p$-values. The difference between trajectories is significant up to the $7^{th}$ extrapolated step and degradation from $8$.
    }
    \label{fig:adni_test}
    \vspace*{-2em}
\end{figure}

Given high resolution 3D images, we would like to evaluate whether our model is able to learn the distribution of mixed effects and perform calibration for personalized prediction. Similar to TADPOLE, we conduct an interpolation experiment and provide representative samples of brain images in Figure~\ref{fig:adni}. We find by 
inspecting the axial/sagital/coronal views
that our model yields meaningful brain images.%
To evaluate the extrapolation capability, given limited number of time steps (only $3$), we perform a statistical test. Recall that our method explicitly models the mixed effect term inside the trajectory to learn the data hierarchy. If our method works as intended, there should be a statistical difference between latent space of trajectories for normal and diseased/abnormal groups. Ideally, this difference should be preserved for several more extrapolation steps. To check this, 
we train our model on $3$ time points. During testing, we use $3$ observed time points for calibration and extrapolate for $5$ more time steps: we get latent trajectories defined for $8$ ($3$ interpolation and $5$ extrapolation) time points. Finally, the resultant trajectories are used to evaluate differences between normal and abnormal groups via a permutation test. The resultant distribution of distances and $p$-values for interpolation and extrapolation is in Fig. ~\ref{fig:adni_test}. As expected, for interpolation and some steps of extrapolation (up to step 7) differences between trajectories is significant (with $p$-value $\leq 0.1$), and becomes less significant with more extrapolation steps. %

\section{Discussions and Conclusions}
We proposed a novel ME-NODE model that enables us to incorporate both fixed and random effects for analyzing the dynamics of panel data. 
Our evaluations on several different tasks show that the ME-NODE loss function can be trained using existing ODE solvers in a  stable and efficient manner. We see various benefits  
from incorporating mixed effects, 
\textbf{(1)} model explicitly learns the correlation structure of the data during the training, which  improves prediction accuracy in setups where samples can be grouped by some criteria;
\textbf{(2)} in contrast to generative models where only initial point is sampled from the distributions, by fixing initial point $z_0$ we can still provide uncertainty of the predictions, because from one initial point we can sample different trajectories;
\textbf{(3)} since our model learns random effects for individual $i$, it allows personalized prediction, given a short history of data, which is useful in biomedical or scientific applications with a limited number of time points per individual along trajectory. One limitation of our approach is that there is not an explicit noise handling mechanism for test time calibration and prediction. This is problematic in large scale high dimensional settings. For example, say that the encoding distribution $\mathcal{N}(\mu,\Sigma)$ produces a small fraction of noisy trajectories. Even in small noise settings, filtering them  for robust personalized prediction requires solving complex optimization problem \citep{bakshi2021list} and so handling noise is especially an open problem in real time, edge deployments.
The code is available at 
\url{https://github.com/vsingh-group/panel_me_ode}.

\section*{Acknowledgments}

This work was supported by 
NIH grants RF1 AG059312 and RF1	AG062336. 
SNR was supported by UIC start-up funds. We thank Seong Jae Hwang for sharing code 
and describing the experiments in 
\cite{hwang2019conditional}.

\bibliography{ref}

\newpage
\appendix

\twocolumn[{%
 \centering
 \LARGE APPENDIX
}]

\section{Proofs}

\subsection{Derivation for final loss}
\begin{equation*}
\begin{split}
\log p(x)\geq & \mathbb{E}_{q(\z, \mathbf{w})}\log p\left(x | \z, \mathbf{w}\right)-KL(q(\z, \mathbf{w}) \| p(\z, \mathbf{w})) \\
\approx&\frac{1}{M} \sum_{m=1}^M \log p\left(x | \z, \mathbf{w}\right)\cdot \1[\obs{\z}_{-0}]{\z_{-0} | z_{0}^m,\mathbf{w}^m} +\\
&\frac{1}{M} \sum_{m=1}^M \log\frac{q(\z, \mathbf{w})}{p(\z, \mathbf{w})}\cdot \1[\obs{\z}_{-0}]{\z_{-0} | z_{0}^m,\mathbf{w}^m}\\
=&\frac{1}{M} \sum_{m=1}^M \left(\log p\left(x | \z, \mathbf{w}\right) + \log\frac{q(\z, \mathbf{w})}{p(\z, \mathbf{w})}\right)\cdot\\
&~~~~~~~~~~~~~~~\1[\obs{\z}_{-0}]{\z_{-0} | z_{0}^m,\mathbf{w}^m} \\
=&\frac{1}{|S|} \sum_{s\in S} \left(\log p\left(x | \z^s, \mathbf{w}^s\right) + \log\frac{q(z_0^s)q(\mathbf{w}^s)}{p(\z^s, \mathbf{w^s})}\right),\\
\end{split}
\end{equation*}
where $S$ is a set, such that  $\{\forall s \in S: \1[\obs{\z}_{-0}]{\z_{-0} | z_{0}^s,\mathbf{w}^s} = 1\}$  and $|S|$ is its size.

\section{Experiments}
\subsection{Rotating MNIST}

\subsubsection{Ability to capture different angles}
In Figure~\ref{fig:supp_mnist_best_vs_mean} we provide visualization of 2 samples with the same digit style, but 2 different angles of rotation through interpolation and extrapolation.

\begin{figure*}[!ht]
    \centering
    \includegraphics[width=\textwidth, trim={1cm, 9.5cm, 2cm, 9.9cm}, clip]{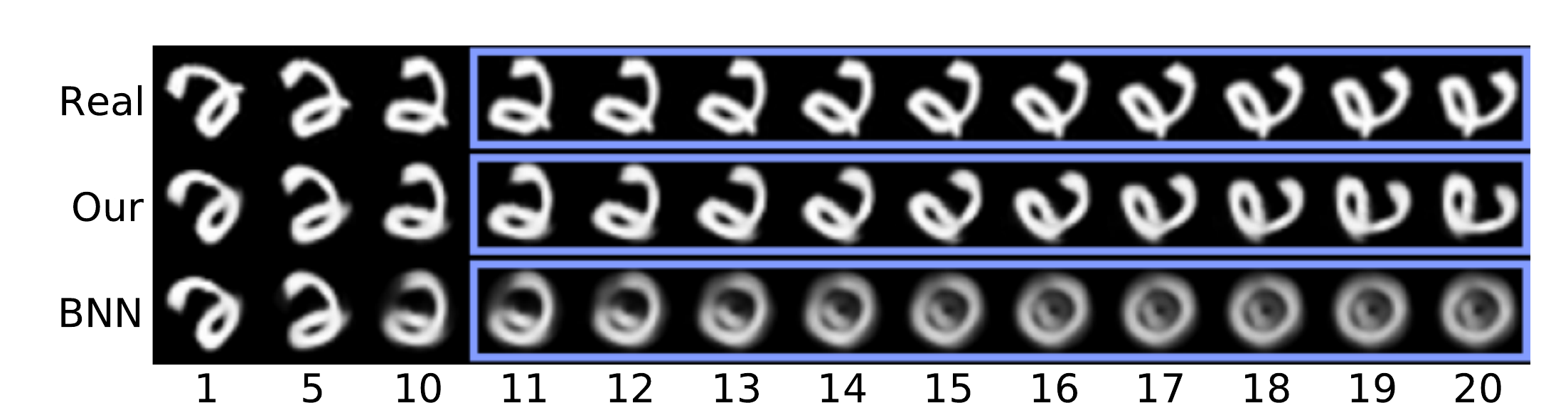}
    \includegraphics[width=\textwidth, trim={1cm, 9.5cm, 2cm, 9.9cm}, clip]{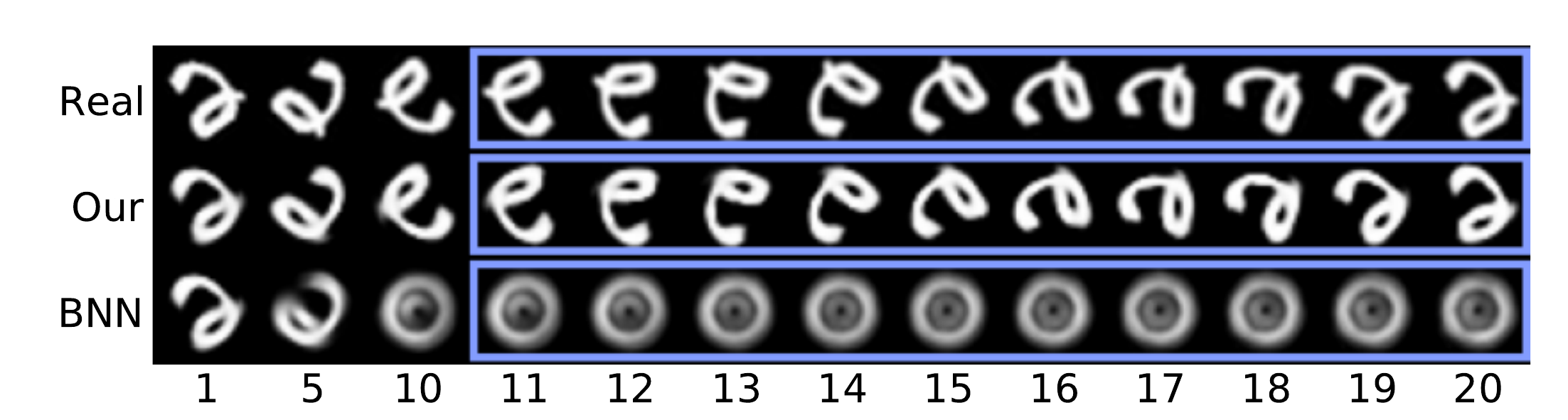}
    \caption{Visualization for two samples from data set with 8 possible angles (slow rotation--top, fast rotation -- bottom). Because of space limitation we show only 3 time steps of interpolation(1, 5, 10) and all steps of extrapolation (11-20) -- blue frame.  We show the calibration effect of our model on extrapolation, compare to BNN. 
    We see that with slow rotation interpolation on all 10 steps for bnn is a little worse than our, but still sensible, while extrapolation is not good anymore. Same time, for fast rotated data even for interpolation BNN provides worse results, and very bad for extrapolation. While our method is good for both.}
    \label{fig:supp_mnist_best_vs_mean}
\end{figure*}

\subsection{ADNI}
Following ADNI setup from the main paper, in Figures \ref{fig:supp_adni_side} and \ref{fig:supp_adni_top}, we provide another samples of our model, comparing with BNN. To evaluate the result visually, we provide a difference between real and prediction, for both our method and BNN. We see that our model gives much better results.

\begin{figure*}[!ht]
    \centering
    \includegraphics[width=0.3\textwidth]{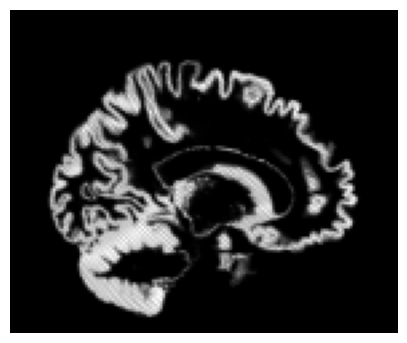}
    \includegraphics[width=0.3\textwidth]{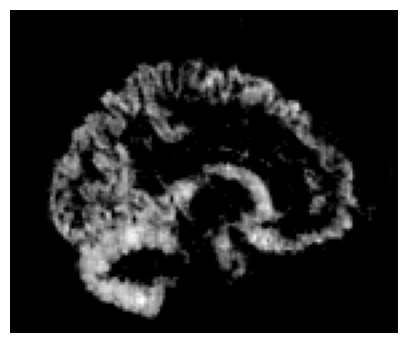}
    \includegraphics[width=0.3\textwidth]{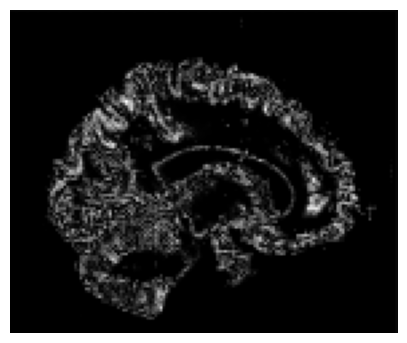}

    \includegraphics[width=0.3\textwidth]{figs/adni/adni_supplement_side/pos_000_02_0_real_0.png}
    \includegraphics[width=0.3\textwidth]{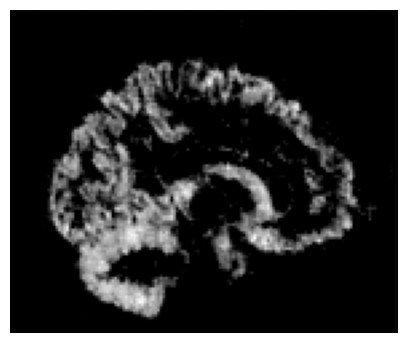}
    \includegraphics[width=0.3\textwidth]{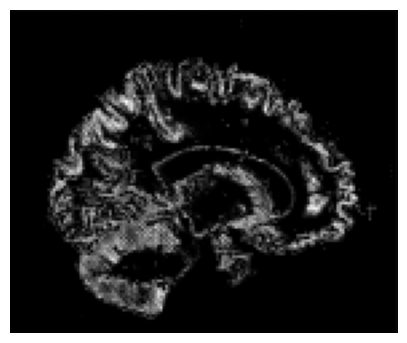}
    
    \caption{\textbf{Top:} Our method, \textbf{Bottom:} BNN. \textbf{From left:} Truth, prediction, difference between truth and prediction. According to difference (3 column), our method performs better than BNN.}
    \label{fig:supp_adni_side}
\end{figure*}
\begin{figure*}[!ht]
    \centering
    \includegraphics[width=0.3\textwidth]{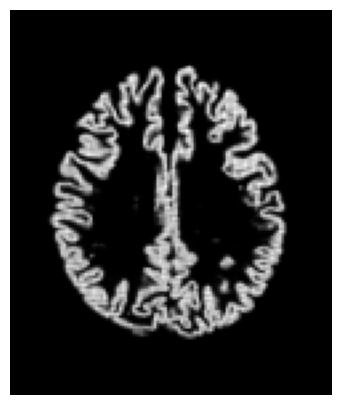}
    \includegraphics[width=0.3\textwidth]{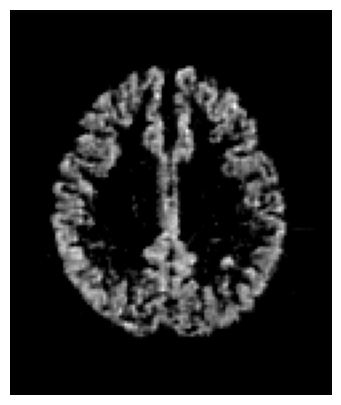}
    \includegraphics[width=0.3\textwidth]{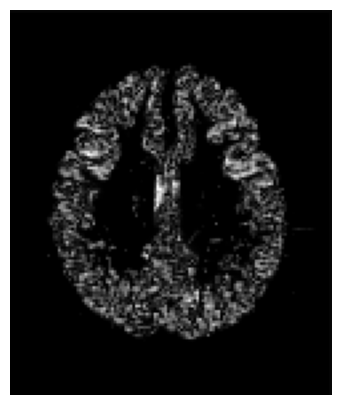}

    \includegraphics[width=0.3\textwidth]{figs/adni/adni_supplement_top/pos_000_02_0_real_0.png}
    \includegraphics[width=0.3\textwidth]{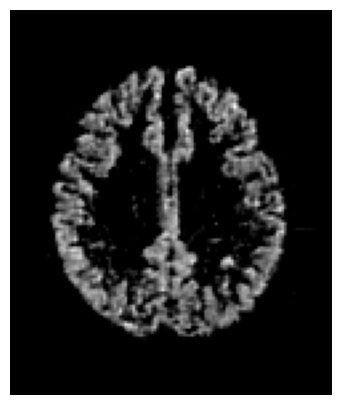}
    \includegraphics[width=0.3\textwidth]{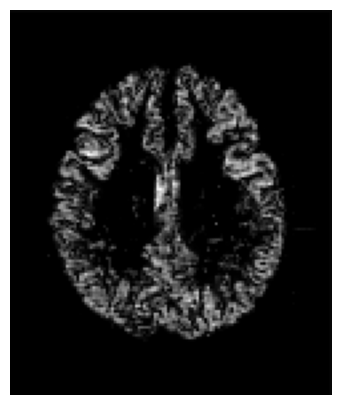}
    
      \caption{\textbf{Top:} Our method, \textbf{Bottom:} BNN. \textbf{From left:} Truth, prediction, difference between truth and prediction. According to difference (3 column), our method performs better than BNN.}
      \label{fig:supp_adni_top}
\end{figure*}

\subsection{Hardware specifications and architecture of networks}
All experiments were executed on NVIDIA - 2080ti, and detailed code will be provided in github repository later.

For Rotating MNIST (2d data) encoder/decoder is described in Figure \ref{fig:2d_encoder} and for ADNI (3d data) encoder is described in Figure \ref{fig:3d_encoder} and decoder in Figure \ref{fig:3d_decoder}.

\begin{figure*}
\centering
\begin{lstlisting}[language=Python, 
        showstringspaces=false,
        formfeed=newpage,
        tabsize=4,
        commentstyle=\itshape,
        basicstyle=\footnotesize\ttfamily,
        morekeywords={sampler_normal, torch},
        frame=lines,
        title={Encoder structure used in ROTATING MNIST},
        label={}
        ]
        
encoder = nn.Sequential(
        nn.Conv2d(input_dim, 12, ks,
                  stride=1, padding=1),
        nn.ReLU(),
        nn.Conv2d(12, 24, ks,
                  stride=2, padding=1), 
        nn.ReLU(),
        nn.Conv2d(24, output_dim, ks,
                  stride=2, padding=1),
        nn.Flatten(2),
        nn.Linear(49, 1), 
        nn.Flatten(1)  
        )
\end{lstlisting}
\begin{lstlisting}[language=Python, 
        showstringspaces=false,
        formfeed=newpage,
        tabsize=4,
        commentstyle=\itshape,
        basicstyle=\footnotesize\ttfamily,
        morekeywords={sampler_normal, torch},
        frame=lines,
        title={Decoder structure used in ROTATING MNIST},
        label={}
        ]      
extend_to_2d = nn.Linear(input_dim, 
                         49 * input_dim)
decoder = nn.Sequential(
        nn.ConvTranspose2d(input_dim,
                           24,
                           ks,
                           stride=2,
                           padding=1,
                           output_padding=1),
        nn.ConvTranspose2d(24,
                           12,
                           ks,
                           stride=2,
                           padding=1,
                           output_padding=1), 
        nn.ConvTranspose2d(12, output_dim, ks, 
                           stride=1, padding=1),
        nn.Sigmoid(),
        )
\end{lstlisting}
\caption{Description of Encoder and Decoder used in experiment with 2d data structure: Rotating MNIST.}
\label{fig:2d_encoder}
\end{figure*}

\begin{figure*}
\begin{lstlisting}[language=Python, 
        showstringspaces=false,
        formfeed=newpage,
        tabsize=4,
        commentstyle=\itshape,
        basicstyle=\footnotesize\ttfamily,
        morekeywords={sampler_normal, torch},
        frame=lines,
        title={Encoder structure used in ADNI},
        label={}
        ]
encoder = nn.Sequential(
        nn.Conv3d(input_dim,
                  8,
                  kernel_size=3,
                  stride=1,
                  padding=1,
                  bias=False),
        nn.ReLU(),
        nn.MaxPool3d(kernel_size=2, return_indices=True),
        nn.Conv3d(8,
                  16,
                  kernel_size=3,
                  stride=1,
                  padding=2,
                  bias=False),
        nn.ReLU(),
        nn.MaxPool3d(kernel_size=2, return_indices=True),
        nn.Conv3d(16,
                  32,
                  kernel_size=3,
                  stride=1,
                  padding=1,
                  bias=False),
        nn.ReLU(),
        nn.MaxPool3d(kernel_size=2, return_indices=True),
        nn.Conv3d(32,
                  64,
                  kernel_size=3,
                  stride=1,
                  padding=2,
                  bias=False),
        nn.ReLU(),
        nn.MaxPool3d(kernel_size=2, return_indices=True),
        nn.Conv3d(64,
                  128,
                  kernel_size=3,
                  stride=1,
                  padding=1,
                  bias=False),
        nn.ReLU(),
        nn.MaxPool3d(kernel_size=2, return_indices=True),
        nn.Conv3d(128,
                  256,
                  kernel_size=3,
                  stride=1,
                  padding=1,
                  bias=False),
        nn.ReLU(),
        nn.MaxPool3d(kernel_size=2, return_indices=True),
        nn.Conv3d(256,
                  output_dim,
                  kernel_size=3,
                  stride=1,
                  padding=1,
                  bias=False),
        nn.ReLU(),
        nn.Flatten(2),
        nn.Linear(8, 1),
        nn.Flatten(1) 
        )
\end{lstlisting}
\caption{Encoder for ADNI}
\label{fig:3d_encoder}
\end{figure*}

\begin{figure*}

\begin{lstlisting}[language=Python, 
        showstringspaces=false,
        formfeed=newpage,
        tabsize=4,
        commentstyle=\itshape,
        basicstyle=\footnotesize\ttfamily,
        morekeywords={sampler_normal, torch},
        frame=lines,
        title={Decoder structure used in ADNI},
        label={}
        ]
extend_to_3d = nn.Linear(input_dim, 2 * 2 * 2 * input_dim)
decoder = nn.Sequential(
        nn.ConvTranspose3d(in_channels=input_dim,
                           out_channels=256,
                           kernel_size=3,
                           padding=1),
        nn.ReLU(),
        nn.MaxUnpool3d(kernel_size=2),
        nn.ConvTranspose3d(in_channels=256,
                           out_channels=128,
                           kernel_size=3,
                           padding=1),
        nn.ReLU(),
        nn.MaxUnpool3d(kernel_size=2),
        nn.ConvTranspose3d(in_channels=128,
                           out_channels=64,
                           kernel_size=3,
                           padding=1),
        nn.ReLU(),
        nn.MaxUnpool3d(kernel_size=2),
        nn.ConvTranspose3d(in_channels=64,
                           out_channels=32,
                           kernel_size=3,
                           padding=2),
        nn.ReLU(),
        nn.MaxUnpool3d(kernel_size=2),
        nn.ConvTranspose3d(in_channels=32,
                           out_channels=16,
                           kernel_size=3,
                           padding=1),
        nn.ReLU(),
        nn.MaxUnpool3d(kernel_size=2),
        nn.ConvTranspose3d(in_channels=16,
                           out_channels=8,
                           kernel_size=3,
                           padding=2),
        nn.ReLU(),
        nn.MaxUnpool3d(kernel_size=2),
        nn.ConvTranspose3d(in_channels=8,
                           out_channels=output_dim,
                           kernel_size=3,
                           padding=1),
        nn.ReLU()
       )
        
\end{lstlisting}
\caption{Decoder for ADNI}
\label{fig:3d_decoder}
\end{figure*}

\end{document}